\documentclass[a4paper]{article}

\usepackage{chngcntr} 
\usepackage[numbers]{natbib} 
\usepackage{subfigure} 
\usepackage[english]{babel}
\usepackage[utf8x]{inputenc}
\usepackage[T1]{fontenc} 

\usepackage[a4paper,top=3cm,bottom=2cm,left=3cm,right=3cm,marginparwidth=1.75cm]{geometry}
\usepackage[colorinlistoftodos]{todonotes}
\usepackage[colorlinks=true, allcolors=blue]{hyperref}
 
\usepackage{enumitem}

\usepackage{authblk}
\usepackage{amsthm}
\usepackage{setspace}
\usepackage{amsfonts}
\usepackage{bm}
\usepackage{xspace}
\usepackage{dsfont}
\usepackage{algorithm,rotating}
\usepackage{graphicx} 
\usepackage{url}
\usepackage{xcolor,soul}
\usepackage{thm-restate}
\usepackage[small,bf]{caption}
\usepackage{standalone}
\usepackage{verbatim}
\usepackage{setspace}
\usepackage{mdwmath}
\usepackage{amssymb}
\usepackage{amsmath}
\usepackage{xfrac}
\usepackage{mathtools}  
\usepackage{tabulary}
\usepackage{booktabs}
\usepackage{epsfig}
\usepackage{wrapfig}
\usepackage{lipsum}
\usepackage{caption}
\usepackage{multirow}
\usepackage{algcompatible} 
\usepackage{footnote}
\usepackage{hyperref}

\newtheorem{theorem}{Theorem}

\newtheorem{definition}{Definition}

\newtheorem{lemma}{Lemma}

\newtheorem{proposition}{Proposition}
\newtheorem{observation}{Observation}

\newtheorem{asm}{Assumption}

\newcommand{\indi}{\mathds{1}}

\newcommand{\tr}{\operatorname{tr}}
\newcommand{\R}{\mathbb{R}}

\newcommand{\var}{\operatorname{Var}}

\newcommand{\vecx}{\mathbf{x}}
\newcommand{\vecw}{\mathbf{w}}

\newcommand{\tvecw}{\widetilde{\mathbf{w}}}

\newcommand{\vecv}{\mathbf{v}}

\newcommand{\vecu}{\mathbf{u}}
\newcommand{\vecy}{\mathbf{y}}
\newcommand{\vecg}{\widehat{\mathbf{g}}}

\newcommand{\vece}{\mathbf{e}}
\newcommand{\vecz}{\mathbf{z}}
\newcommand{\vecp}{\mathbf{p}}

\newcommand{\matG}{\mathbf{G}}

\newcommand{\matH}{\mathbf{H}}
\newcommand{\matU}{\mathbf{U}}
\newcommand{\matQ}{\mathbf{Q}}
\newcommand{\matI}{\mathbf{I}}

\newcommand{\gradf}{\nabla f}
\newcommand{\hessf}{\nabla^2 f}
\newcommand{\gradh}{\nabla h}

\newcommand{\gradF}{\nabla F}
\newcommand{\hessF}{\nabla^2 F}

\newcommand{\setS}{\mathcal{S}}

\newcommand{\bigo}{\mathcal{O}}

\newcommand{\W}{\mathcal{W}}
\newcommand{\vect}[1]{\mathbf{#1}}
\newcommand{\mat}[1]{\mathbf{#1}}

\newcommand{\twonm}[1]{\left\|#1\right\|_2}
\newcommand{\twonms}[1]{\|#1\|_2}

\newcommand{\fbnorms}[1]{\|#1\|_F}

\newcommand{\EE}{\ensuremath{\mathbb{E}}}
\newcommand{\EXPS}[1]{\mathbb{E}[#1]}
\newcommand{\innerp}[2]{\left\langle#1,#2\right\rangle}
\newcommand{\innerps}[2]{\langle#1,#2\rangle}

\newcommand{\B}{\mathcal{B}}
\newcommand{\Ball}{\mathbb{B}}
\newcommand{\D}{\mathcal{D}}

\newcommand{\med}{{\sf med}}
\newcommand{\trim}{{\sf trmean}}
\newcommand{\tm}{{\sf tm}}
\newcommand{\Z}{\mathcal{Z}}
\newcommand{\vecmu}{\boldsymbol\mu}
\newcommand{\vecdlt}{\boldsymbol\delta}

\newcommand{\dev}{\mathrm{d}}

\newcommand{\vol}{\mathrm{Vol}}
\newcommand{\Tth}{T_{\mathrm{th}}}
\newcommand{\hF}{\widehat{F}}

\newcommand{\hvecx}{\widehat{\mathbf{x}}}
\newcommand{\A}{\mathcal{A}}

\algblockdefx{PARFOR}{ENDPARFOR}[1]%
  {\textbf{for all }#1 \textbf{do in parallel}}%
  {\textbf{end for}}

\title{Defending Against Saddle Point Attack in Byzantine-Robust Distributed Learning}

\author[1]{Dong Yin \thanks{dongyin@berkeley.edu}}
\author[3]{Yudong Chen \thanks{yudong.chen@cornell.edu}}
\author[1]{Kannan Ramchandran \thanks{kannanr@berkeley.edu}}
\author[1,2]{Peter Bartlett \thanks{peter@berkeley.edu}}
\affil[1]{Department of Electrical Engineering and Computer Sciences, UC Berkeley}
\affil[2]{Department of Statistics, UC Berkeley}
\affil[3]{School of Operations Research and Information Engineering, Cornell University}   

\begin{document}
\maketitle

\begin{abstract}
We study robust distributed learning that involves minimizing a non-convex loss function with saddle points. We consider the Byzantine setting where some worker machines have abnormal or even arbitrary and adversarial behavior. In this setting, the Byzantine machines may create fake local minima near a saddle point that is far away from any true local minimum, even when robust gradient estimators are used. We develop \emph{ByzantinePGD}, a robust first-order algorithm that can provably escape saddle points and fake local minima, and converge to an approximate true local minimizer with low iteration complexity. As a by-product, we give a simpler algorithm and analysis for escaping saddle points in the usual non-Byzantine setting. We further discuss three robust gradient estimators that can be used in ByzantinePGD, including median, trimmed mean, and iterative filtering. We characterize their performance in concrete statistical settings, and argue for their near-optimality in low and high dimensional regimes. 
\end{abstract}

\section{Introduction}\label{sec:intro}

Distributed computing becomes increasingly important in modern data-intensive applications. In many applications, large-scale datasets are distributed over multiple machines for parallel processing in order to speed up computation. In other settings, the data sources are naturally distributed, and for privacy and efficiency considerations, the data are not transmitted to a central machine. An example is the recently proposed \emph{Federated Learning} paradigm~\cite{mcmahan2017federated,konevcny2016federated,konevcny2015federated}, in which the data are stored and processed locally in end users' cellphones and personal computers.

In a standard worker-server distributed computing framework, a single master machine is in charge of maintaining and updating the parameter of interest, and a set of worker machines store the data, perform local computation and communicate with the master. In this setting, messages received from worker machines are prone to errors due to data corruption, hardware/software malfunction, and communication delay and failure. These problems are only exacerbated in a decentralized distributed architecture such as Federated Learning, where some machines may be subjected to malicious and coordinated attack and manipulation. A well-established framework for studying such scenarios is the \emph{Byzantine} setting~\cite{lamport1982byzantine}, where a subset of machines behave completely arbitrarily---even in a way that depends on the algorithm used and the data on the other machines---thereby capturing the unpredictable nature of the errors. Developing distributed algorithms that are robust in the Byzantine setting has become increasingly critical.

In this paper we focus on robust distributed optimization for statistical learning problems. Here the data points are generated from some unknown distribution $\D$ and stored locally in $m$ worker machines, each storing $n$ data points; the goal is to minimize a population loss function $F:\W\rightarrow \R$ defined as an expectation over $\D$, where $\W\subseteq\R^d$ is the parameter space. We assume that $\alpha \in (0, 1/2)$ fraction of the worker machines are Byzantine; that is, their behavior is arbitrary. This Byzantine-robust distributed learning problem has attracted attention in a recent line of work~\cite{alistarh2018sgd,blanchard2017byzantine,chen2017distributed,feng2014distributed,su2016fault,su2016non,yin2018byzantine}. This body of work develops robust algorithms that are guaranteed to output an approximate minimizer of $F$ when it is convex, or an approximate stationary point in the non-convex case.

However, fitting complicated machine learning models often requires finding a \emph{local minimum} of \emph{non-convex} functions, as exemplified by training deep neural networks and other high-capacity learning architectures~\cite{soudry2016no,ge2016matrix,ge2017no}.
It is well-known that many of the stationary points of these problems are in fact saddle points and far away from any local minimum~\cite{kawaguchi2016deep,ge2017no}. These tasks hence require algorithms capable of efficiently \emph{escaping saddle points} and converging approximately to a local minimizer. In the centralized setting without Byzantine adversaries, this problem has been studied actively and recently ~\cite{ge2015escaping,jin2017escape,carmon2016accelerated,jin2017accelerated}.

A main observation of this work is that the interplay between non-convexity and Byzantine errors makes escaping saddle points much more challenging. In particular, 
by orchestrating their messages sent to the master machine, \emph{the Byzantine machines can create fake local minima near a saddle point of $ F $ that is far away from any true local minimizer}. Such a strategy, which may be referred to as \textbf{saddle point attack}, foils existing algorithms as we elaborate below:

\begin{itemize}[leftmargin=3mm]
	\item \textbf{Challenges due to non-convexity:} When $F$ is convex, gradient descent (GD) equipped with a robust gradient estimator is guaranteed to find an approximate global minimizer (with accuracy depending on the fraction of Byzantine machines)~\cite{chen2017distributed,yin2018byzantine,alistarh2018sgd}.
    However, when $F$ is non-convex, such algorithms may be trapped in the neighborhood of a saddle point; see Example 1 in Appendix~\ref{apx:hardness}.

	\item \textbf{Challenges due to Byzantine machines:} Without Byzantine machines, vanilla GD~\cite{lee2016converge}, as well as its more efficient variants such as perturbed gradient descent (PGD)~\cite{jin2017escape}, are known to converge to a local minimizer with high probability.
    However, Byzantine machines can manipulate PGD and GD (even robustified) into fake local minimum near a saddle point; see Example 2 in Appendix~\ref{apx:hardness}.
\end{itemize}

We discuss and compare with existing work in more details in Section~\ref{sec:related}. The observations above show that existing robust and saddle-escaping algorithms, as well as their naive combination, are insufficient against saddle point attack. Addressing these challenges requires the development of new robust distributed optimization algorithms.

\subsection{Our Contributions} 
In this paper, we develop \textbf{\emph{ByzantinePGD}}, a computation- and communication-efficient first-order algorithm that is able to escape saddle points and the fake local minima created by Byzantine machines, and converge to an approximate local minimizer of a non-convex loss. To the best of our knowledge, our algorithm is the first to achieve such guarantees under \emph{adversarial} noise. 

Specifically, ByzantinePGD aggregates the empirical gradients received from the normal and Byzantine machines, and computes a robust estimate $ \vecg(\vecw) $ of the true gradient $ \gradF(\vecw) $ of the population loss $ F $. Crucial to our algorithm is the injection of random perturbation to the iterates $ \vecw $, which serves the dual purpose of escaping saddling point and fake local minima. Our use of perturbation thus plays a more signified role than in existing algorithms such as PGD~\cite{jin2017escape}, as it also serves to combat the effect of Byzantine errors. To achieve this goal, we incorporate two crucial innovations: (i) we use multiple rounds of larger, yet carefully calibrated, amount of perturbation that is necessary to survive saddle point attack, (ii) we use the moving distance in the parameter space as the criterion for successful escape, eliminating the need of (robustly) evaluating function values. Consequently, our analysis is significantly different, and arguably simpler, than that of PGD.

We develop our algorithmic and theoretical results in a flexible, two-part framework, decomposing the \emph{optimization} and \emph{statistical} components of the problem. 

\paragraph*{The optimization part:} We consider a general problem of optimizing a population loss function $ F $ given an \emph{inexact gradient oracle}. For each query point~$\vecw$, the $\Delta$-inexact gradient oracle returns a vector $ \vecg(\vecw) $ (possibly chosen adversarially) that satisfies $\twonms{\vecg(\vecw) - \gradF(\vecw)} \le \Delta$, where $\Delta$ is non-zero but bounded. Given access to such an inexact oracle, we show that ByzantinePGD outputs an approximate local minimizer; moreover, \emph{no} other algorithm can achieve significantly better performance in this setting in terms of the dependence on $ \Delta $ and the dimension of the parameter space $d$:

\begin{theorem}[Informal; see Sec.~\ref{sec:convergence}]\label{thm:informal_main}
Within $\widetilde{\bigo}(\frac{1}{\Delta^2})$ iterations, ByzantinePGD outputs an approximate local minimizer $\tvecw$ that satisfies $\twonms{\gradF(\tvecw)} \lesssim \Delta $ and $\lambda_{\min}\big( \hessF(\tvecw) \big) \gtrsim -\Delta^{2/5}d^{1/5}$, where $ \lambda_{\min} $ is the minimum eigenvalue. In addition, given only access to $\Delta$-inexact gradient oracle, \textbf{no} algorithm is guaranteed to find a point $ \tvecw $ with $\twonms{\gradF(\tvecw)} < \Delta/2 $ \textbf{or} $\lambda_{\min}\big( \hessF(\tvecw) \big) > -\Delta^{1/2}/2$.
\end{theorem}

Our algorithm is communication-efficient: it only sends gradients, and the number of parallel iterations in our algorithm \emph{matches} the well-known iteration complexity of GD for non-convex problems in non-Byzantine setting~\cite{nesterov1998introductory} (up to log factors). In the exact gradient setting, a variant of the above result in fact matches the guarantees for PGD~\cite{jin2017escape}---as mentioned, our proof is simpler.

Additionally, beyond Byzantine distributed learning, our results apply to any non-convex optimization problems (distributed or not) with inexact information for the gradients, including those with noisy but non-adversarial gradients. Thus, we believe our results are of independent interest in broader settings.
	
\paragraph*{The statistical part:} The optimization guarantee above can be applied whenever one has a robust aggregation procedure that serves as an inexact gradient oracle with a bounded error $\Delta$. We consider three concrete examples of such robust procedures: median, trimmed mean, and iterative filtering~\cite{diakonikolas2016robust,diakonikolas2017being}. Under statistical settings for the data, we provide explicit bounds on their errors $ \Delta $ as a function of the number of worker machines $m$, the number of data points on each worker machine $n$, the fraction of Byzantine machines $\alpha$, and the dimension of the parameter space $d$. Combining these bounds with the optimization result above, we obtain concrete statistical guarantees on the output $ \tvecw $. Furthermore, we argue that our first-order guarantees on $\twonms{\gradF(\tvecw)}$ are often nearly \emph{optimal} when compared against a universal statistical lower bound. This is summarized below:

\begin{theorem}[Informal; see Sec.~\ref{sec:robust_estimation}]\label{thm:statistical}
	When combined with each of following three robust aggregation procedures, ByzantinePGD achieves the statistical guarantees: \\
	(i) median/; $\twonms{\gradF(\tvecw)} \lesssim \frac{\alpha \sqrt{d}}{\sqrt{n}} + \frac{d}{\sqrt{nm}} + \frac{\sqrt{d}}{n} $;\\ (ii) trimmed mean: $\twonms{\gradF(\tvecw)} \lesssim \frac{\alpha d}{\sqrt{n}} + \frac{d}{\sqrt{nm}}$;\\ (iii) iterative filtering: $\twonms{\gradF(\tvecw)} \lesssim \frac{\sqrt{\alpha}}{\sqrt{n}} + \frac{\sqrt{d}}{\sqrt{nm}}$. \\	
	Moreover, \textbf{no} algorithm can achieve $  \twonms{\gradF(\tvecw)} = o\big(  \frac{\alpha}{\sqrt{n}} + \frac{\sqrt{d}}{\sqrt{nm}} \big)$.
\end{theorem}

We emphasize that the above results are established under a very strong adversary model: the Byzantine machines are allowed to send messages that depend arbitrarily on each other and on the data on the normal machines; they may even behave adaptively during the iterations of our algorithm. Consequently, this setting requires robust \emph{functional} estimation (of the gradient function), which is a much more challenging problem than the robust \emph{mean} estimation setting considered by existing work on median, trimmed mean and iterative filtering. To overcome this difficulty,  we make use of careful covering net arguments to establish certain error bounds that hold \emph{uniformly} over the parameter space, regardless of the behavior of the Byzantine machines. Importantly, our inexact oracle framework allows such arguments to be implemented in a transparent and modular manner.

\paragraph*{Notation} For an integer $N>0$, define the set $[N] := \{1,2,\ldots, N\}$. For matrices, denote the operator norm by $\twonms{\cdot}$; for symmetric matrices, denote the largest and smallest eigenvalues by $\lambda_{\max}(\cdot)$ and $\lambda_{\min}(\cdot)$, respectively. The $d$-dimensional $\ell_2$ ball centered at $\vecw$ with radius $r$ is denoted by $\Ball_{\vecw}^{(d)}(r)$, or $\Ball_{\vecw}(r)$ when it is clear from the context. 

\section{Related Work}\label{sec:related}

\begin{table*}[ht]
\centering
\begin{tabular}{|c|c|c|c|c|}
\hline
 Algorithm   &  PGD & Neon+GD  & Neon2+GD  &  \textbf{ByzantinePGD}  \\  \hline
 Byzantine-robust? & no & no & no  & yes \\  \hline
 Purpose of perturbation & escape SP & escape SP & escape SP &  \shortstack{escape SP \\ \& robustness}  \\ \hline
 Escaping method  & GD & NC search & NC search &  inexact GD \\  \hline
 Termination criterion  &  decrease in $F$ & decrease in $F$ & distance in $\W$ & distance in $\W$ \\  \hline
 Multiple rounds? & no & no & no & yes  \\  \hline
\end{tabular}
\caption{Comparison with PGD, Neon+GD, and Neon2+GD. SP = saddle point.}
\label{tab:comparison_perturbation}
\end{table*}

\begin{table*}[ht]
\centering
\begin{tabular}{|c|c|c|}
\hline
   &  Robust Aggregation Method  &  Non-convex Guarantee  \\  \hline
\citet{feng2014distributed} & geometric median & no  \\  \hline
\citet{chen2017distributed}  &  geometric median  &  no   \\  \hline
\citet{blanchard2017byzantine} & Krum &  first-order  \\  \hline
\citet{yin2018byzantine}  & median, trimmed mean   &  first-order  \\ \hline
\citet{xie2018generalized} & mean-around-median, marginal median & first-order   \\  \hline
\citet{alistarh2018sgd} & martingale-based & no  \\ \hline
\citet{su2018securing}  & iterative filtering  & no  \\  \hline
\textbf{This work} &  median, trimmed mean, iterative filtering & second-order  \\ \hline
\end{tabular}
\caption{Comparison with other Byzantine-robust distributed learning algorithms.}
\label{tab:comparison_byzantine}
\end{table*}

\paragraph*{Efficient first-order algorithms for escaping saddle points} Our algorithm is related to a recent line of work which develops efficient first-order algorithms for escaping saddle points. Although vanilla GD converges to local minimizers almost surely~\cite{lee2016converge,lee2017first}, achieving convergence in polynomial time requires more a careful algorithmic design~\cite{du2017gradient}. Such convergence guarantees are enjoyed by several GD-based algorithms; examples include PGD~\cite{jin2017escape}, Neon+GD~\cite{xu2017first}, and Neon2+GD~\cite{allen2017neon2}. The general idea of these algorithms is to run GD and add perturbation to the iterate when the gradient is small. While our algorithm also uses this idea, the design and analysis techniques of our algorithm are significantly different from the work above in the following aspects (also summarized in Table~\ref{tab:comparison_perturbation}).

\begin{itemize}[leftmargin=3mm]

\item In our algorithm, besides helping with escaping saddle points, the random perturbation has the additional role of defending against adversarial errors.

\item The perturbation used in our algorithm needs to be larger, yet carefully calibrated, in order to account for the influence of the inexactness of gradients across the iterations, especially iterations for escaping saddle points. 

\item We run inexact GD after the random perturbation, while Neon+GD and Neon2+GD use negative curvature (NC) search. It is not immediately clear whether NC search can be robustified against Byzantine failures. Compared to PGD, our analysis is arguably \emph{simpler} and more \emph{straightforward}. 

\item Our algorithm does not use the value of the loss function (hence no need for robust function value estimation); PGD and Neon+GD assume access to the (exact) function values.

\item We employed \emph{multiple} rounds of perturbation to boost the probability of escaping saddle points; this technique is not used in PGD, Neon+GD, or Neon2+GD.

\end{itemize}

\paragraph*{Inexact oracles} Optimization with an inexact oracle (e.g.\ noisy gradients) has been studied in various settings such as general convex optimization~\cite{bertsekas2000errors,devolder2014inexact}, robust estimation~\cite{prasad2018robust}, and structured non-convex problems~\cite{balakrishnan2014EM,chen2015fast,candes2014wirtinger,zhang2016provable}. 
Particularly relevant to us is the recent work by~\citet{jin2018minimizing}, who consider the problem of minimizing $F$ when only given access to the gradients of another \emph{smooth} function $\hF$ satisfying $\|{\nabla\hF(\vecw) - \gradF(\vecw)}\|_\infty \le \Delta/\sqrt{d},~\forall\vecw$. Their algorithm uses Gaussian smoothing on $\hF$. We emphasize that the inexact gradient setting considered by them is much more benign than our Byzantine setting, since (i) their inexactness is defined in terms of $\ell_\infty$ norm whereas the inexactness in our problem is in $\ell_2$ norm, and (ii) we assume that the inexact gradient can be \emph{any} vector within $\Delta$ error, and thus the smoothing technique is not applicable in our problem. Moreover, the iteration complexity obtained by~\citet{jin2018minimizing} may be a high-degree polynomial of the problem parameters and thus not  suitable for distributed implementation.

\paragraph*{Byzantine-robust distributed learning} Solving large scale learning problems in distributed systems has received much attention in recent years, where communication efficiency and Byzantine robustness are two important topics~\cite{shamir2014communication,lee2015distributed,yin2017gradient,blanchard2017byzantine,chen2018draco,damaskinos2018asynchronous}. Here, we compare with existing Byzantine-robust distributed learning algorithms that are most relevant to our work, and summarize the comparison in Table~\ref{tab:comparison_byzantine}. A general idea of designing Byzantine-robust algorithms is to combine optimization algorithms with a robust aggregation (or outlier removal) subroutine. For convex losses, the aggregation subroutines analyzed in the literature include geometric median~\cite{feng2014distributed,chen2017distributed}, median and trimmed mean~\cite{yin2018byzantine}, iterative filtering for the high dimensional setting~\cite{su2018securing}, and martingale-based methods for the SGD setting~\cite{alistarh2018sgd}. For non-convex losses, to the best of our knowledge, existing works only provide first-order convergence guarantee (i.e., small gradients), by using aggregation subroutines such as the Krum function~\cite{blanchard2017byzantine}, median and trimmed mean~\cite{yin2018byzantine}, mean-around-median and marginal median~\cite{xie2018generalized}. In this paper, we make use of subroutines based on median, trimmed mean, and iterative filtering. Our analysis of median and trimmed mean follows~\citet{yin2018byzantine}. Our results based on the iterative filtering subroutine, on the other hand, are new:

\begin{itemize}[leftmargin=3mm]

\item The problem that we tackle is harder than what is considered in the original iterative filtering papers~\cite{diakonikolas2016robust,diakonikolas2017being}. There they only consider robust estimation of a single mean parameter, where as we guarantee robust gradient estimation over the parameter space. 

\item Recent work by~\citet{su2018securing} also makes use of the iterative filtering subroutine for the Byzantine setting. They only study strongly convex loss functions, and assume that the gradients are sub-exponential and  $ d \le \bigo(\sqrt{mn})$. Our results apply to the non-convex case and do not require the aforementioned condition on $d$ (which may therefore scale, for example, linearly with the sample size $ mn $), but we impose the stronger assumption of sub-Gaussian gradients. 

\end{itemize}

\paragraph*{Other non-convex optimization algorithms} Besides first-order GD-based algorithms, many other non-convex optimization methods that can provably converge to approximate local minimum have received much attention in recent years. For specific problems such as phase retrieval~\cite{candes2014wirtinger}, low-rank estimation~\cite{chen2015fast,zhao2015nonconvex}, and dictionary learning~\cite{agarwal2014learning,sun2015complete}, many algorithms are developed by leveraging the particular structure of the problems, and the either use a smart initialization~\cite{candes2014wirtinger,tu2015low} or initialize randomly~\cite{chen2018gradient,chatterji2017alternating}. Other algorithms are developed for general non-convex optimization, and they can be classified into gradient-based~\cite{ge2015escaping,levy2016power,xu2017first,allen2017natasha,allen2017neon2,jin2017accelerated}, Hessian-vector-product-based~\cite{carmon2016accelerated,agarwal2016linear,royer2018complexity,royer2018newton}, and Hessian-based~\cite{nesterov2006cubic,curtis2017trust} methods. While algorithms using Hessian information can usually achieve better convergence rates---for example, $\bigo(\frac{1}{\epsilon^{3/2}})$ by~\citet{curtis2017trust}, and $\bigo(\frac{1}{\epsilon^{7/4}})$ by~\citet{carmon2016accelerated}--- 
gradient-based methods are easier to implement in practice, especially in the distributed setting we are interested in. 

\paragraph*{Robust statistics} Outlier-robust estimation is a classical topic in statistics~\cite{huber2011robust}. The coordinate-wise median aggregation subroutine that we consider is related to the median-of-means estimator~\cite{nemirovskii1983problem,jerrum1986random}, which has been applied to various robust inference problems~\cite{minsker2015geometric,lugosi2016risk,minsker2017distributed}.

A recent line of work develops efficient robust estimation algorithms in high-dimensional settings~\cite{bhatia2015robust,diakonikolas2016robust,lai2016agnostic,charikar2017learning,steinhardt2017resilience,li2017robust,bhatia2017consistent,klivans2018efficient,liu2018regression}. In the centralized setting, the recent work~\cite{diakonikolas2018sever} proposes a scheme, similar to the iterative filtering procedure, that iteratively removes outliers for gradient-based optimization.

\section{Problem Setup}\label{sec:setup}

We consider empirical risk minimization for a statistical learning problem where each data point $ \vecz $ is sampled from an unknown distribution $\D$ over the sample space $\Z$. Let $f(\vecw; \vecz)$ be the loss function of a parameter vector $\vecw\in\W\subseteq\R^d$, where $\W$ is the parameter space. The population loss function is therefore given by $F(\vecw) := \EE_{\vecz\sim\D}[f(\vecw;\vecz)]$.

We consider a distributed computing system with one \emph{master} machine and $m$ \emph{worker} machines, $\alpha m$ of which are Byzantine machines and the other $ (1-\alpha)m $ are normal. Each worker machine has $n$ data points sampled i.i.d.~from $\D$. Denote by $\vecz_{i,j}$ the $j$-th data point on the $i$-th worker machine, and let $F_i(\vecw) := \frac{1}{n}\sum_{j=1}^n f(\vecw; \vecz_{i,j}) $ be the empirical loss function on the $i$-th machine. The master machine and worker machines can send and receive messages via the following communication protocol: In each parallel iteration, the master machine sends a parameter vector $\vecw$ to all the worker machines, and then each \emph{normal} worker machine computes the gradient of its empirical loss $F_i(\cdot)$ at $\vecw$ and sends the gradient to the master machine. The Byzantine machines may be jointly controlled by an adversary and send arbitrary or even malicious messages. We denote the unknown set of Byzantine machines  by $\B$, where  $|\B|=\alpha m$. With this notation, the gradient sent by the $i$-th worker machine is

\begin{equation}\label{eq:def_grad_value_pair}
 \vecg_i(\vecw) = \begin{cases}
 \gradF_i(\vecw)  & i \in [m]\setminus\B, \\
 * & i \in \B, \end{cases}
\end{equation}
where the symbol $*$ denotes an arbitrary vector. As mentioned, the adversary is assumed to have complete knowledge of the algorithm used and the data stored on all machines, and the Byzantine machines may collude~\cite{lynch1996distributed} and adapt to the output of the master and normal worker machines. We only make the mild assumption that the adversary cannot \emph{predict} the random numbers generated by the master machine.

We consider the scenario where $F(\vecw)$ is non-convex, 
and our goal to find an approximate local minimizer of $F(\vecw)$. Note that a first-order stationary point (i.e., one with a small gradient) is not necessarily close to a local minimizer, since the point may be a \emph{saddle point} whose Hessian matrix has a large negative eigenvalue. 
Accordingly, we seek to find a \emph{second-order stationary point} $\tvecw$, namely, one with a small gradient and a nearly positive semidefinite Hessian:
\begin{definition}[Second-order stationarity]\label{def:second_stationary}
We say that $\tvecw$ is an $(\epsilon_g, \epsilon_H)$-second-order stationary point of a twice differentiable function $F(\cdot)$ if $\twonms{\gradF(\tvecw)} \le \epsilon_g$ and $\lambda_{\min}\big(\hessF(\tvecw)\big) \ge - \epsilon_H$.
\end{definition}

In the sequel, we make use of several standard concepts from continuous optimization.

\begin{definition}[Smooth and Hessian-Lipschitz functions]\label{def:lipschitz}
A function $ h $ is called $ L $-smooth if $\sup_{\vecw\neq\vecw'} \frac{\twonms{\gradh(\vecw) - \gradh(\vecw')}}{\twonms{\vecw - \vecw'}} \le L $, and $ \rho $-Hessian Lipschitz if $\sup_{\vecw\neq\vecw'} \frac{\twonms{\nabla^2h(\vecw) - \nabla^2h(\vecw')}}{\twonms{\vecw - \vecw'}} \le \rho $.
\end{definition}
Throughout this paper, the above properties are imposed on the \emph{population} loss function $F(\cdot)$. 
\begin{asm}\label{asm:lip_assumptions}
$F$ is $L_F$-smooth, and $\rho_F$-Hessian Lipschitz on $\W$.
\end{asm}

\section{Byzantine Perturbed Gradient Descent}\label{sec:perturbed_bgd}

In this section, we describe our algorithm, \emph{Byzantine Perturbed Gradient Descent} (ByzantinePGD), which provably finds a second-order stationary point of the population loss $F(\cdot)$ in the distributed setting with Byzantine machines. 
As mentioned, ByzantinePGD robustly aggregates gradients from the worker machines, and performs multiple rounds of carefully calibrated perturbation to combat the effect of Byzantine machines. We now elaborate.

It is well-known that naively aggregating the workers' messages using standard averaging can be arbitrarily skewed in the presence of just a single Byzantine machine. In view of this, we introduce the subroutine $\mathsf{GradAGG}\{ \vecg_i(\vecw) \}_{i=1}^m$, which robustly aggregates the gradients $\{ \vecg_i(\vecw) \}_{i=1}^m$ collected from the $ m $ workers. We stipulate that $\mathsf{GradAGG}$  provides an estimate of the true population gradient $ \gradF(\cdot) $ with accuracy $ \Delta $, \emph{uniformly} across $\W$. This  property is formalized using the terminology of \emph{inexact gradient oracle}.

\begin{definition}[Inexact gradient oracle]\label{def:inexact_oracle}
We say that $\mathsf{GradAGG}$ provides a $\Delta$-inexact gradient oracle for the population loss $F(\cdot)$ if, for every $\vecw\in\W$, we have 
$ \twonms{\mathsf{GradAGG}\{ \vecg_i(\vecw) \}_{i=1}^m - \gradF(\vecw)} \le \Delta$.
\end{definition}

Without loss of generality, we assume that $\Delta \le 1$ throughout the paper. In this section, we treat $\mathsf{GradAGG}$ as a given black box; in Section~\ref{sec:robust_estimation}, we discuss several robust aggregation algorithms and characterize their inexactness $\Delta$. We emphasize that in the Byzantine setting, the output of $\mathsf{GradAGG}$ can take values adversarially within the error bounds; that is, $\mathsf{GradAGG}\{ \vecg_i(\vecw)\}_{i=1}^m$ may output an arbitrary vector in the ball $\Ball_{\gradF(\vecw)}(\Delta)$, and this vector can depend on the data in all the machines and all previous iterations of the algorithm.

The use of robust aggregation with bounded inexactness, however, is not yet sufficient to guarantee convergence to an approximate local minimizer. As mentioned, the Byzantine machines may create fake local minima that traps a vanilla gradient descent iteration. Our ByzantinePGD algorithm is designed to escape such fake minima as well as any existing saddle points of $ F $.

\subsection{Algorithm}\label{sec:algorithm}

We now describe the details of our algorithm, given in the left panel of Algorithm~\ref{alg:main_alg}. We focus on unconstrained optimization, i.e., $\W = \R^d$. In Section~\ref{sec:robust_estimation}, we show that the iterates $\vecw$ during the algorithm actually stay in a bounded $\ell_2$ ball centered at the initial iterate $\vecw_0$, and we will discuss the statistical error rates within the bounded space. 

In each parallel iteration, the master machine sends the current iterate $\vecw$ to all the worker machines, and the worker machines send back $\{\vecg_i(\vecw)\}$. The master machine aggregates the workers' gradients using $\mathsf{GradAGG}$ and computes a robust estimate $\vecg(\vecw)$ of the population gradient $\gradF(\vecw)$. The master machine then performs a gradient descent step using $\vecg(\vecw)$. 
This procedure is repeated until it reaches a point $\tvecw$ with $\twonms{ \vecg(\vecw) } \le \epsilon$ for a pre-specified threshold $\epsilon$. 

At this point, $\tvecw$ may lie near a saddle point whose Hessian has a large negative eigenvalue. To escape this potential saddle point, the algorithm invokes the $\mathsf{Escape}$ routine (right panel of Algorithm~\ref{alg:main_alg}), which performs $Q$ rounds of \emph{perturbation-and-descent} operations. 
In each round, the master machine perturbs $\tvecw$ randomly and independently within the ball $\Ball_{\tvecw}(r)$. Let $\vecw_0^\prime$ be the perturbed vector. Starting from the $\vecw_0^\prime$, the algorithm conducts at most $\Tth$ parallel iterations of $\Delta$-inexact gradient descent (using $\mathsf{GradAGG}$ as before):
\begin{equation}\label{eq:perturbed_iterates}
\vecw_t^\prime = \vecw_{t-1}^\prime - \eta \vecg(\vecw_{t-1}^\prime),~t \le \Tth.
\end{equation}
During this process, once we observe that $\twonm{\vecw_{t}^\prime - \vecw_{0}^\prime} \ge R$ for some pre-specified threshold $R$ (this means the iterate moves by a sufficiently large distance in the parameter space), we claim that $\tvecw$ is a saddle point and the algorithm has escaped it; we then resume $\Delta$-inexact gradient descent starting from $\vecw_{t}^\prime$. If after $Q$ rounds no sufficient move in the parameter space is ever observed, we claim that $\tvecw$ is a second-order stationary point of $F(\vecw)$ and output $\tvecw$. 

\begin{figure*}
\noindent\begin{minipage}[t]{\textwidth}
\centering%
\setlength{\fboxsep}{0.5pt}
\fbox{%
\begin{minipage}[t]{.48\textwidth}
     $\mathsf{ByzantinePGD}(\vecw_0, \eta, \epsilon, r,  Q, R, \Tth)$
     \begin{algorithmic}
     \STATE $\vecw \leftarrow \vecw_0$
   \WHILE{ $\mathrm{true}$ }
   \STATE \textit{\underline{Master}}: send $\vecw$ to worker machines.
   \PARFOR{$i\in[m]$}
   \STATE \textit{\underline{Worker $i$}}: compute $ \vecg_i(\vecw) $
   \STATE send to master machine.
   \ENDPARFOR
   \STATE \textit{\underline{Master}}: 
   \STATE $\vecg(\vecw) \leftarrow \mathsf{GradAGG}\{ \vecg_i(\vecw) \}_{i=1}^m$.
   \IF{$ \twonms{\vecg(\vecw)} \le \epsilon$}
   \STATE \textit{\underline{Master}}: $\tvecw \leftarrow \vecw$,
   \STATE $(\mathsf{esc}, \vecw, \vecg(\vecw))$  $\leftarrow$ $\mathsf{Escape}$ ($\tvecw$, $\eta$, $r$, $Q$, $R$, $\Tth$).
   \IF{$ \mathsf{esc} = \mathrm{false}$}
   \STATE \textbf{return} $\tvecw$.
   \ENDIF
   \ENDIF
   \STATE \textit{\underline{Master}}: $\vecw \leftarrow \vecw - \eta\vecg(\vecw) $.
   \ENDWHILE
   \vspace{0.05in}
\end{algorithmic}
\end{minipage}
}
\fbox{
   \begin{minipage}[t]{.48\textwidth}
$\mathsf{Escape}(\tvecw, \eta, r, Q, R, \Tth)$
     \begin{algorithmic}
     \FOR{$k = 1,2, \ldots, Q$}
\STATE \textit{\underline{Master}}: sample $ \vecp_k\sim\text{Unif}(\Ball_0(r))$,
\STATE $\vecw' \leftarrow \tvecw + \vecp_k$, $\vecw_0^\prime \leftarrow \vecw'$. 
\FOR{$t = 0, 1, \ldots, \Tth$}
 \STATE \textit{\underline{Master}}: send $\vecw'$ to worker machines.
 \PARFOR{$i\in[m]$}
   \STATE \textit{\underline{Worker $i$}}: compute $ \vecg_i(\vecw') $ \STATE send to master machine.
 \ENDPARFOR
\STATE \textit{\underline{Master}}: $\vecg(\vecw') \leftarrow \mathsf{GradAGG}\{ \vecg_i(\vecw') \}_{i=1}^m$.
\IF{$ \twonms{\vecw' - \vecw_0^\prime} \ge R$} \STATE \textbf{return} $(\mathrm{true}, \vecw', \vecg(\vecw'))$.
 \ELSE
 \STATE $\vecw' \leftarrow \vecw' - \eta\vecg(\vecw') $
 \ENDIF
\ENDFOR
\ENDFOR
\STATE \textbf{return} $(\mathrm{false}, \vecw', \vecg(\vecw'))$.
\vspace{0.18in}
     \end{algorithmic}
   \end{minipage}
}
   \captionof{algorithm}{Byzantine Perturbed Gradient Descent (ByzantinePGD)}
   \label{alg:main_alg}
\end{minipage}
\end{figure*}

\subsection{Convergence Guarantees}\label{sec:convergence}

In this section, we provide the theoretical result guaranteeing that Algorithm~\ref{alg:main_alg} converges to a second-order stationary point. In Theorem~\ref{thm:main}, we let $F^* := \min_{\vecw\in\R^d} F(\vecw)$, $\vecw_0$ be the initial iterate, and $F_0 := F(\vecw_0)$.

\begin{theorem}[ByzantinePGD]\label{thm:main}
Suppose that Assumptions~\ref{asm:lip_assumptions} holds, and assume that $\mathsf{GradAGG}$ provides a $ \Delta $-inexact gradient oracle for $F(\cdot)$ with $\Delta \le 1$. Given any $\delta\in(0,1)$, choose the parameters for Algorithm~\ref{alg:main_alg} as follows: step-size $\eta = \frac{1}{L_F}$, $\epsilon = 3\Delta$, $r = 4 \Delta^{3/5} d^{3/10}\rho_F^{-1/2}$, $R = \Delta^{2/5} d^{1/5} \rho_F^{-1/2}$, 
\begin{align*}
Q &= 2 \log \bigg( \frac{ \rho_F( F_0 - F^*) }{48 L_F \delta (\Delta^{6/5}d^{3/5} + \Delta^{7/5}d^{7/10}) } \bigg),  \text{ and}\\
\Tth &= \frac{L_F}{384 (\rho_F^{1/2} + L_F) (\Delta^{2/5}d^{1/5} + \Delta^{3/5}d^{3/10}) }.
\end{align*} 
Then, with probability at least $1- \delta$, the output of Algorithm~\ref{alg:main_alg}, denoted by $\tvecw$, satisfies the bounds
\begin{equation} \label{eq:first_second_order} 
\begin{aligned}
\twonms{\gradF(\tvecw)} & \le 4\Delta,  \\
\lambda_{\min} \big(\hessF (\tvecw)\big) & \ge  -1900 \big( \rho_F^{1/2} +  L_F \big) \Delta^{2/5}d^{1/5} \log \Big( \frac{10}{\Delta} \Big), 
\end{aligned}
\end{equation}
and the algorithm terminates within $\frac{2(F_0 - F^*)L_F}{3\Delta^2} Q$ parallel iterations.
\end{theorem}

We prove Theorem~\ref{thm:main} in Appendix~\ref{prf:main}.\footnote{We make no attempt in optimizing the multiplicative constants in Theorem~\ref{thm:main}.} Below let us parse the above theorem and discuss its implications. Focusing on the scaling with $ \Delta $ and $d$, we may read off from Theorem~\ref{thm:main} the following result:
\begin{observation}\label{obs:achievable}
Under the above setting, within $\widetilde{\bigo} (\frac{1}{\Delta^2})$ parallel iterations, ByzantinePGD outputs an $(\bigo(\Delta), \widetilde{\bigo}(\Delta^{2/5}d^{1/5}))$-second-order stationary point $ \tvecw $ of $F(\cdot)$;\footnote{Here, by using the symbol $\widetilde{\bigo}$, we ignore logarithmic factors and only consider the dependence on $\Delta$.} that is,
 $$\twonms{\gradF(\tvecw)} \le 4\Delta \text{~~and~~} \lambda_{\min}(\hessF(\tvecw)) \ge -\widetilde{\bigo}(\Delta^{2/5}d^{1/5}).$$
\end{observation}

In terms of the iteration complexity, it is well-known that for a smooth non-convex $F(\cdot)$, gradient descent requires at least  $\frac{1}{\Delta^2}$ iterations to achieve $\twonms{\gradF(\tvecw)}\le \bigo(\Delta)$~\cite{nesterov1998introductory}; up to logarithmic factors, our result matches this complexity bound. In addition, our $\bigo(\Delta)$ first-order guarantee is clearly order-wise \emph{optimal}, as the gradient oracle is $\Delta$-inexact. It is currently unclear to us whether our $\widetilde{\bigo}(\Delta^{2/5}d^{1/5})$ second-order guarantee is optimal. We provide a converse result showing that one cannot hope to achieve a second-order guarantee better than $ {\bigo}(\Delta^{1/2})$.

\begin{proposition}\label{obs:second_order_lb}
There exists a class of real-valued $1$-smooth and $1$-Hessian Lipschitz differentiable functions $\mathcal{F}$ such that, for any algorithm that only uses a $\Delta$-inexact gradient oracle, there exists $f\in\mathcal{F}$ such that the output of the algorithm $\tvecw$ must satisfy $\twonms{\gradF(\tvecw)} > \Delta/2$ and $\lambda_{\min}(\hessF(\tvecw)) < -\Delta^{1/2}/2$.
\end{proposition}

We prove Proposition~\ref{obs:second_order_lb} in Appendix~\ref{prf:second_order_lb}. Again, we emphasize that our results above are in fact not restricted to the Byzantine distributed learning setting. They apply to any non-convex optimization problems (distributed or not) with inexact information for the gradients, including those with noisy but non-adversarial gradients; see Section~\ref{sec:related} for comparison with related work in such settings. 

As a byproduct, we can show that with a different choice of parameters, ByzantinePGD can be used in the standard (non-distribued) setting with access to the \emph{exact} gradient $\gradF(\vecw)$, and the algorithm converges to an $(\epsilon, \widetilde{\bigo}(\sqrt{\epsilon}))$-second-order stationary point within $\bigo(\frac{1}{\epsilon^2})$ iterations:

\begin{theorem}[Exact gradient oracle]\label{thm:exact_oracle}
Suppose that Assumptions~\ref{asm:lip_assumptions} holds, and assume that for any query point $\vecw$ we can obtain exact gradient, i.e., $\vecg(\vecw) \equiv\gradF(\vecw)$. For any $\epsilon \in (0, \min\{ \frac{1}{\rho_F}, \frac{4}{L_F^2\rho_F} \})$ and $\delta\in(0,1)$, we choose the parameters in Algorithm~\ref{alg:main_alg} as follows: step-size $\eta = 1/L_F$, $Q=1$, $r=\epsilon$, and $R = \sqrt{\epsilon/\rho_F}$, $\Tth = \frac{L}{12\rho_F(R+r)}$. Then, with probability at least $1-\delta$, Algorithm~\ref{alg:main_alg} outputs a $\tvecw$ satisfying the bounds
\begin{align*}
\twonms{\gradF(\tvecw)} \le & \epsilon, \\
\lambda_{\min}(\hessF(\tvecw)) \ge & -60\sqrt{\rho_F \epsilon} \log \Big(\frac{8\rho_F\sqrt{d}(F_0 - F^*)}{\delta\epsilon^2} \Big),
\end{align*}
and the algorithm terminates within $\frac{2L_F(F_0 - F^*)}{\epsilon^2}$ iterations.
\end{theorem}
We prove Theorem~\ref{thm:exact_oracle} in Appendix~\ref{apx:exact_oracle}. The convergence guarantee above matches that of the original PGD algorithm~\cite{jin2017escape} up to logarithmic factors. Moreover, our proof is considerably simpler,  and our algorithm only requires gradient information, whereas the original PGD algorithm also needs function values.

\section{Robust Estimation of Gradients}\label{sec:robust_estimation}

The results in the previous section can be applied as long as one has a robust aggregation subroutine $\mathsf{GradAGG}$ that provides a $\Delta$-inexact gradient oracle of the population loss $ F $. In this section, we discuss three concrete examples of $\mathsf{GradAGG}$: \emph{median}, \emph{trimmed mean}, and a high-dimension robust estimator based on the \emph{iterative filtering} algorithm~\cite{diakonikolas2016robust,diakonikolas2017being,steinhardt2017resilience}. We characterize their inexactness $\Delta$ under the statistical setting in Section~\ref{sec:setup}, where the data points are sampled independently according to an unknown distribution~$\D$.

To describe our statistical results, we need the standard notions of sub-Gaussian/exponential random vectors.
\begin{definition}[sub-Gaussianity and sub-exponentiality]\label{def:subexp}
A random vector $\vecx$ with mean $\vecmu$ is said to be  $\zeta$-sub-Gaussian if $ \EXPS{\exp(\lambda \innerps{\vecx - \vecmu}{\vecu})} \le e^{\frac{1}{2}\zeta^2\lambda^2 \twonms{\vecu}^2}, \forall~\lambda, \vecu $. It is said to be $\xi$-sub-exponential if $\EXPS{\exp(\lambda \innerps{\vecx - \vecmu}{\vecu})} \le e^{\frac{1}{2}\xi^2\lambda^2 \twonms{\vecu}^2},~\forall~|\lambda| < \frac{1}{\xi},\vecu$.
\end{definition}

We also need the following result (proved in Appendix~\ref{prf:size_w_space}), which shows that the iterates of ByzantinePGD in fact stay in a bounded set around the initial iterate $\vecw_0$.

\begin{proposition}\label{ppn:size_w_space}
Under the choice of algorithm parameters in Theorem~\ref{thm:main}, all the iterates $\vecw$ in ByzantinePGD stay in the $\ell_2$ ball $\Ball_{\vecw_0}(D/2)$ with $D := C\frac{F_0 - F^*}{\Delta}$, where $C>0$ is a number that only depends on $L_F$ and $\rho_F$. 
\end{proposition}
Consequently, for the convergence guarantees of ByzantinePGD to hold, we only need $\mathsf{GradAGG}$ to satisfy the inexact oracle property (Definition~\ref{def:inexact_oracle}) within the bounded set $\W = \Ball_{\vecw_0}(D/2)$, with $D$ given in Proposition~\ref{ppn:size_w_space}. As shown below, the three aggregation procedures indeed satisfy this property, with their inexactness $ \Delta $ depends mildly (logarithmically) on the radius $D$.

\subsection{Iterative Filtering Algorithm}

We start with a recently developed high-dimension robust estimation technique called the \emph{iterative filtering} algorithm~\cite{diakonikolas2016robust,diakonikolas2017being,steinhardt2017resilience} and use it to build the subroutine $\mathsf{GradAGG} $. As can be seen below, iterative filtering can tolerate a constant fraction of Byzantine machines even when the dimension grows---an advantage over simpler algorithms such as median and trimmed mean.

We relegate the details of the iterative filtering algorithm to Appendix~\ref{apx:filtering}. Again, we emphasize that the original iterative filtering algorithm is proposed to robustly estimate a single parameter vector, whereas in our setting, since the Byzantine machines may produce unspecified probabilistic dependency across the iterations, we need to prove an error bound for robust gradient estimation uniformly across the parameter space $\W$. We prove such a bound for iterative filtering under the following two assumptions on the gradients and the smoothness of each loss function $f(\cdot;\vecz)$.

\begin{asm}\label{asm:sub_guass_grad}
For each $\vecw \in \W$, $\gradf(\vecw; \vecz)$ is $\zeta$-sub-Gaussian.
\end{asm}
\begin{asm}\label{asm:lip_each_loss_2}
For each $\vecz\in\Z$, $f(\cdot; \vecz)$ is $L$-smooth.
\end{asm}
Let $\mat{\Sigma}(\vecw)$ be the covariance matrix of $\gradf(\vecw;\vecz)$, and define $\sigma:=\sup_{\vecw\in\W}\twonms{\mat{\Sigma}(\vecw)}^{1/2}$. We have the following bounds on the inexactness parameter of iterative filtering.

\begin{theorem}[Iterative Filtering]\label{thm:iterative_filtering}
Suppose that Assumptions~\ref{asm:sub_guass_grad} and~\ref{asm:lip_each_loss_2} hold. Use the iterative filtering algorithm described in Appendix~\ref{apx:filtering} for $\mathsf{GradAGG}$, and assume that $\alpha \le \frac{1}{4}$. With probability $1-o(1)$, $\mathsf{GradAGG}$ provides a $\Delta_{\mathsf{ftr}}$-inexact gradient oracle with
\[
\Delta_{\mathsf{ftr}} \le c \left( (\sigma + \zeta)\sqrt{\frac{\alpha}{n}} + \zeta\sqrt{\frac{d}{nm}} \right)\sqrt{\log(nmDL)},
\]
where $c$ is an absolute constant.
\end{theorem}
The proof of Theorem~\ref{thm:iterative_filtering} is given in Appendix~\ref{prf:iterative_filtering}. Assuming bounded $ \sigma$ and $ \zeta $, we see that iterative filtering provides an $\widetilde{\bigo} \big( \sqrt{\frac{\alpha}{n}} + \sqrt{\frac{d}{nm}} \big) $-inexact gradient oracle.

\subsection{Median and Trimmed Mean}

The median and trimmed mean operations are two widely used robust estimation methods. While the dependence of their performance on $d$ is not optimal, they are conceptually simple and computationally fast, and still have good performance in low dimensional settings. We apply these operations in a coordinate-wise fashion to build $\mathsf{GradAGG}$. 

Formally, for a set of vectors $\vecx^i \in \mathbb{R}^d$, $i \in [m]$, their coordinate-wise median $\vecu := \med\{\vecx^i\}_{i=1}^m$ is a vector with its $ k $-th coordinate being $u_k = \med\{x^i_k\}_{i=1}^m$ for each $k\in[d]$, where $ \med $ is the usual (one-dimensional) median. The coordinate-wise $ \beta $-trimmed mean $\vecu := \trim_\beta\{\vecx^i\}_{i=1}^m$ is a vector with $u_k = \frac{1}{(1-2\beta)m}\sum_{x \in \mathcal{U}_k}x$ for each $k\in[d]$, where $\mathcal{U}_k$ is a subset of $\{x^1_k,\ldots, x^m_k\}$ obtained by removing the largest and smallest $\beta$ fraction of its elements.

For robust estimation of the gradient in the Byzantine setting, the error bounds of median and trimmed mean have been studied by~\citet{yin2018byzantine}. For completeness, we record their results below as an informal theorem; details are relegated to Appendix~\ref{apx:med_tm}.

\begin{theorem}[Informal]\label{thm:med_tm}
~\cite{yin2018byzantine} Under appropriate smoothness and probabilistic assumptions,\footnote{Specifically, for median we assume that gradients have bounded skewness, and for trimmed mean we assume that the gradients are sub-exponentially distributed.} with high probability, the median operation provides a $\Delta_{\med}$-inexact gradient oracle with $\Delta_{\med} \lesssim \frac{\alpha \sqrt{d}}{\sqrt{n}} + \frac{d}{\sqrt{nm}} + \frac{\sqrt{d}}{n} $, and the trimmed mean operation provides a $\Delta_{\tm}$-inexact gradient oracle with $\Delta_{\tm} \lesssim \frac{\alpha d}{\sqrt{n}} + \frac{d}{\sqrt{nm}}$.
\end{theorem}

\subsection{Comparison and Optimality}\label{sec:comparison}
In Table~\ref{tab:comparison}, we compare the above three algorithms in terms of the dependence of their gradient inexactness $\Delta$ on the problem parameters $\alpha$, $n$, $m$, and $d$ . We see that when  $d = \bigo(1)$, the median and trimmed mean algorithms have better inexactness due to a better scaling with $\alpha$. When $d$ is large, iterative filtering becomes preferable. 

\begin{table}[h]
\centering
\begin{tabular}{|c|c|}
\hline
        & Gradient inexactness $\Delta$  \\  \hline
median & $\widetilde{\bigo}\big(\frac{\alpha \sqrt{d}}{\sqrt{n}} + \frac{d}{\sqrt{nm}} + \frac{\sqrt{d}}{n}\big)$  \\ \hline
trimmed mean & $ \widetilde{\bigo}\big(\frac{\alpha d}{\sqrt{n}} + \frac{d}{\sqrt{nm}}\big) $  \\ \hline
iterative filtering & $ \widetilde{\bigo}\big(\frac{\sqrt{\alpha}}{\sqrt{n}} + \frac{\sqrt{d}}{\sqrt{nm}}\big) $   \\ \hline
\end{tabular}
\caption{Statistical bounds on gradient inexactness $\Delta$.} 
\label{tab:comparison}
\end{table}

Recall that according to Observation~\ref{obs:achievable}, with $\Delta$-inexact gradients the ByzantinePGD algorithm converges to an $(\bigo(\Delta), \widetilde{\bigo}(\Delta^{2/5}))$-second-order stationary point. Combining this general result with the bounds in Table~\ref{tab:comparison}, we obtain explicit statistical guarantees on the output of ByzantinePGD. To understand the statistical optimality of these guarantees, we provide a converse result below.

\begin{observation}\label{obs:lower_bound}
There exists a statistical learning problem in the Byzantine setting such that the output $\tvecw$ of \emph{any} algorithm must satisfy $\twonms{\gradF(\tvecw)} = \Omega\big(\frac{\alpha}{\sqrt{n}} + \frac{\sqrt{d}}{\sqrt{nm}}\big)$ with a constant probability.
\end{observation}  
We prove Observation~\ref{obs:lower_bound} in Appendix~\ref{apx:lower_bound}. In view of this observation, we see that in terms of the first-order guarantee (i.e., on $\twonms{\gradF(\tvecw)}$) and up to logarithmic factors, trimmed mean is optimal if $d=\bigo(1)$, the median is optimal if $d = \bigo(1)$ and $n \gtrsim m$, and iterative filtering is optimal if $\alpha = \Theta(1)$. The statistical optimality of their second-order guarantees (i.e., on $ \lambda_{\min} (\hessF(\tvecw)) $) is currently unclear to us, and we believe this is an interesting problem for future investigation.

\section{Conclusion}\label{sec:conclusion}

In this paper, we study security issues that arise in large-scale distributed learning because of the presence of saddle points in non-convex loss functions. We observe that in the presence of non-convexity and Byzantine machines, escaping saddle points becomes much more challenging. We develop ByzantinePGD, a computation- and communication-efficient algorithm that is able to provably escape saddle points and converge to a second-order stationary point, even in the presence of Byzantine machines. We also discuss three different choices of the robust gradient and function value aggregation subroutines in ByzantinePGD---median, trimmed mean, and the iterative filtering algorithm. We characterize their performance in statistical settings, and argue for their near-optimality in different regimes including the high dimensional setting.

\subsection*{Acknowledgements}
D. Yin is partially supported by Berkeley DeepDrive Industry Consortium. Y. Chen is partially supported by NSF CRII award 1657420 and grant 1704828. K. Ramchandran is partially supported by NSF CIF award 1703678. P. Bartlett is partially supported by NSF grant IIS-1619362. The authors would like to thank Zeyuan Allen-Zhu for pointing out a potential way to improve our initial results, and Ilias Diakonikolas for discussing references~\cite{diakonikolas2016robust,diakonikolas2017being,diakonikolas2018sever}.

\bibliographystyle{plainnat}
\bibliography{ref}

\appendix
\section*{Appendix}

\section{Challenges of Escaping Saddle Points in the Adversarial Setting} \label{apx:hardness}

We provide two examples showing that in non-convex setting with saddle points, inexact oracle can lead to much worse sub-optimal solutions than in the convex setting, and that in the adversarial setting, escaping saddle points can be inherently harder than the adversary-free case.

Consider standard gradient descent using exact or $\Delta$-inexact gradients. Our first example shows that Byzantine machines have a more severe impact in the non-convex case than in the convex case. 

\paragraph{Example 1.}\label{ex:convexity} Let $d=1$ and consider the functions $F^{(1)}(w) = (w-1)^2$ and $F^{(2)}(w) = (w^2-1)^2/4$.
Here $F^{(1)}$ is strongly convex with a unique local minimizer $w^*=1$, whereas $F^{(2)}$ has two local (in fact, global) minimizers $w^*=\pm 1$ and a saddle point (in fact, a local maximum) $w=0$. Proposition~\ref{ppn:bad_example_2} below shows the following: for the convex $F^{(1)}$, gradient descent (GD) finds a near-optimal solution with sub-optimality proportional to $\Delta$, regardless of initialization; for the nonconvex $F^{(2)}$, GD initialized near the saddle point $w=0$ suffers from an $\Omega(1)$ sub-optimality gap. 

\begin{proposition}\label{ppn:bad_example_2}
Suppose that $\Delta \le 1/2$. Under the setting above, the following holds.\\
(i) For $F^{(1)}$, starting from any $w_0$, GD using a $\Delta$-inexact gradient oracle finds $w$ with $F^{(1)}(w) - F^{(1)}(w^*) \le \bigo (\Delta )$. \\
(ii) For $F^{(2)}$, there exists an adversarial strategy such that starting from a $w_0$ sampled uniformly from $[-r,r]$, GD with a $\Delta$-inexact gradient oracle outputs $w$ with $F^{(2)}(w) - F^{(2)}(w^*) \ge \frac{9}{64}, \forall w^*=\pm1$, with probability $\min\{1, \frac{\Delta}{r}\}$.
\end{proposition} 

\begin{proof}
Since $F^{(2)}(w) = \frac{1}{4}(w^2-1)^2$, we have $\gradF^{(2)}(w)=w^3 - w$. For any $w\in[-\Delta, \Delta]$, $| \gradF^{(2)}(w) | \le \Delta$ (since $\Delta \le 1/2$). Thus, the adversarial oracle can always output $\widehat{g}(w) = 0$ when $w\in[-\Delta, \Delta]$, and we have $| \widehat{g}(w) - \gradF^{(2)}(w) | \le \Delta$. Thus, if $w\in[-\Delta, \Delta]$, the iterate can no longer move with this adversarial strategy. Then, we have $F^{(2)}(w) - F^{(2)}(w^*) \ge F^{(2)}(\Delta) - 0 \ge \frac{9}{64}$ (since $\Delta \le 1/2$). The result for the convex function $F^{(1)}$ is a direct corollary of Theorem 1 in~\cite{yin2018byzantine}.
\end{proof}

Our second example shows that escaping saddle points is much harder in the Byzantine setting than in the non-Byzantine setting.

\paragraph{Example 2.} Let $d=2$, and assume that in the neighborhood $\Ball_0(b)$ of the origin, $F$ takes the quadratic form
$F(\vecw) \equiv \frac{1}{2}w_1^2 - \frac{\lambda}{2} w_2^2$,
with $\lambda > \epsilon_H$.\footnote{$F(\vecw) \equiv \frac{1}{2}w_1^2 - \frac{\lambda}{2} w_2^2$ holds locally around the origin, not globally; otherwise $F(\vecw)$ has no minimum.} 
The origin $\vecw_0 = 0$ is not an $(\epsilon_g, \epsilon_H)$-second-order stationary point, but rather a saddle point. Proposition~\ref{ppn:bad_example} below shows that exact GD escapes the saddle point almost surely, while GD with an inexact oracle fails to do so.

\begin{proposition}\label{ppn:bad_example}
Under the setting above, if one chooses $r < b$ and sample $\vecw$ from $\Ball_0(r)$ uniformly at random, then: \\
(i) Using exact gradient descent, with probability $1$, the iterate $\vecw$ eventually leaves $\Ball_0(r)$.\\
(ii) There exists an adversarial strategy such that, when we update $\vecw$ using $\Delta$-inexact gradient oracle, if $\Delta \ge \lambda r$, with probability $1$, the iterate $\vecw$ cannot leave $\Ball_0(r)$; otherwise with probability $\frac{2}{\pi}\Big( \arcsin \big(\frac{\Delta}{\lambda r}\big) + \frac{\Delta}{\lambda r}\sqrt{1 - (\frac{\Delta}{\lambda r})^2} \Big)$ the iterate $\vecw$ cannot leave $\Ball_0(r)$.
\end{proposition} 

\begin{proof}
Since $F(\vecw) = \frac{1}{2}w_1^2 - \frac{1}{2}\lambda w_2^2$, $\forall~\vecw\in\Ball_0(r)$, we have $\gradF(\vecw)=[w_1,~-\lambda w_2]^\top$. Sample $\vecw_0$ uniformly at random from $\Ball_0(r)$, and we know that with probability $1$, $w_{0,2} \neq 0$. Then, by running exact gradient descent $\vecw_{t+1} = \vecw_t - \eta \gradF(\vecw_t) $, we can see that the second coordinate of $\vecw_t$ is $w_{t,2} = (1+\eta \lambda)^t w_{0,2}$. When $w_{0,2}$, we know that as $t$ gets large, we eventually have $w_{t,2}>r$, which implies that the iterate leaves $\Ball_0(r)$.

On the other hand, suppose that we run $\Delta$-inexact gradient descent, i.e., $\vecw_{t+1} = \vecw_t - \eta \vecg(\vecw_t)$ with $\twonms{\vecg(\vecw_t) - \gradF(\vecw_t)} \le \Delta$. In the first step, if $| w_{0,2} | \le \frac{\Delta}{\lambda}$, the adversary can simply replace $\gradF(\vecw_0)$ with $\vecg(\vecw_0) = [w_{0,1},~0]^\top$ (one can check that here we have $\twonms{\vecg(\vecw_0) - \gradF(\vecw_0)} \le \Delta$), and then the second coordinate of $\vecw_1$ does not change, i.e., $w_{1,2}=w_{0,2}$. In the following iterations, the adversary can keep using the same strategy and the second coordinate of $\vecw$ never changes, and then the iterates cannot escape $\Ball_0(r)$, since $F(\vecw)$ is a strongly convex function in its first coordinate. To compute the probability of getting stuck at the saddle point, we only need to compute the area of the region $\{\vecw \in \Ball_0(r) : |w_2| \le \frac{\Delta}{\lambda}\}$, which can be done via simple geometry.
\end{proof}

\paragraph*{Remark.} Even if we choose the largest possible perturbation in $\Ball_0(r)$, i.e., sample $\vecw$ from the circle $\{\vecw\in\R^2 : \twonms{\vecw} = r\}$, the stuck region still exists. We can compute the length of the arc $\{ \twonms{\vecw} = r : |w_2| \le \frac{\Delta}{\lambda}\}$ and find the probability of stuck. One can find that when $\Delta \ge \lambda r$, the probability of being stuck in $\Ball_0(r)$ is still $1$, otherwise, the probability of being stuck is $\frac{2}{\pi}( \arcsin(\frac{\Delta}{\lambda r}) ) $.

The above examples show that the adversary can significantly alter the landscape of the function near a saddle point. We counter this by exerting a large perturbation on the iterate so that it escapes this bad region. The amount of perturbation is carefully calibrated to ensure that the algorithm finds a descent direction ``steep'' enough to be preserved under $\Delta$-corruption, while not compromising the accuracy. Multiple rounds of perturbation are performed, boosting the escape probability exponentially.

\section{Proof of Theorem~\ref{thm:main}}\label{prf:main}

We first analyze the gradient descent step with $\Delta$-inexact gradient oracle.

\begin{lemma}\label{lem:each_iter}
Suppose that $\eta = 1/ L_F$. For any $\vecw\in\W$, if we run the following inexact gradient descent step:
\begin{equation}\label{eq:iter_w}
\vecw' = \vecw - \eta\vecg(\vecw),
\end{equation}
with $\twonms{\vecg(\vecw) - \gradF(\vecw)} \le \Delta$. Then, we have
\[
F(\vecw') \le F(\vecw) - \frac{1}{2L_F}\twonms{\gradF(\vecw)}^2 + \frac{1}{2L_F}\Delta^2.
\]
\end{lemma}
\begin{proof}
Since $F(\vecw)$ is $L_F$ smooth, we know that
\begin{align*}
F(\vecw')  \le  &  F(\vecw) + \innerps{\gradF(\vecw)}{ \vecw' - \vecw} + \frac{L_F}{2}\twonms{\vecw' - \vecw}^2 \\
= & F(\vecw) - \innerps{\gradF(\vecw)}{\frac{1}{L_F} (\vecg(\vecw) - \gradF(\vecw))  } - \innerps{\gradF(\vecw)}{\frac{1}{L_F}\gradF(\vecw)} \\
&+ \frac{1}{2L_F}\twonms{\vecg(\vecw) - \gradF(\vecw) + \gradF(\vecw)}^2  \\
\le & F(\vecw) - \frac{1}{2L_F}\twonms{\gradF(\vecw)}^2 + \frac{1}{2L_F}\Delta^2.
\end{align*}
\end{proof}

Let $\epsilon$ be the threshold on $\twonms{\vecg(\tvecw)}$ that the algorithm uses to determine whether or not to add perturbation. Choose $\epsilon = 3\Delta$. Suppose that at a particular iterate $\tvecw$, we observe $\twonms{\vecg(\tvecw)} > \epsilon$. Then, we know that 
\[
\twonms{\gradF(\tvecw)} \ge \twonms{\vecg(\tvecw)} - \Delta \ge 2\Delta.
\]
According to Lemma~\ref{lem:each_iter}, by running one iteration of the inexact gradient descent step, the decrease in function value is at least
\begin{equation}\label{eq:func_decrease_no_perturb2}
\frac{1}{2L_F}\twonms{\gradF(\tvecw)}^2 - \frac{1}{2L_F}\Delta^2 \ge \frac{3\Delta^2}{2L_F}.
\end{equation}

We proceed to analyze the perturbation step, which happens when the algorithm arrives at an iterate $\tvecw$ with $\twonms{\vecg(\tvecw)} \le \epsilon$. 
In this proof, we slightly abuse the notation. Recall that in equation~\eqref{eq:perturbed_iterates} in Section~\ref{sec:algorithm}
, we use $\vecw_t^\prime$ $(0\le t \le \Tth)$ to denote the iterates of the algorithm in the saddle point escaping process. Here, we simply use $\vecw_t$ to denote these iterates.
We start with the definition of \emph{stuck region} at $\tvecw\in\W$.

\begin{definition}\label{def:stuck_region}
Given $\tvecw\in\W$, and parameters $r$, $R$, and $\Tth$, the stuck region $\mathbb{W}_S(\tvecw, r, R, \Tth) \subseteq \Ball_{\tvecw}(r)$ is a set of $\vecw_0 \in \Ball_{\tvecw}(r) $ which satisfies the following property: there exists an adversarial strategy such that when we start with $\vecw_0$ and run $\Tth$ gradient descent steps with $\Delta$-inexact gradient oracle $\vecg(\vecw)$:
\begin{equation}\label{eq:stuck_gd}
\vecw_t = \vecw_{t-1} - \eta \vecg(\vecw_{t-1}),~ t = 1,2,\ldots, \Tth,
\end{equation}
we observe $\twonms{ \vecw_t - \vecw_0 } < R$, $\forall~t \le \Tth$.
\end{definition}

When it is clear from the context, we may simply use the terminology stuck region $\mathbb{W}_S$ at $\tvecw$. The following lemma shows that if $\nabla^2 F(\tvecw)$ has a large negative eigenvalue, then the stuck region has a small width along the direction of the eigenvector associated with this negative eigenvalue.

\begin{lemma}\label{lem:moving_dist}
Assume that the smallest eigenvalue of $\matH := \nabla^2F(\tvecw)$ satisfies $\lambda_{\min}(\matH) \le -\gamma < 0$,
and let the unit vector $\vece$ be the eigenvector associated with $\lambda_{\min}(\matH)$. Let $\vecu_0, \vecy_0 \in \Ball_{\tvecw}(r)$ be two points such that $\vecy_0 = \vecu_0 + \mu_0\vece$ with some $\mu_0 \ge \mu \in (0, r)$. Choose step size $\eta = \frac{1}{L_F}$, and consider the stuck region $\mathbb{W}_S(\tvecw, r, R, \Tth)$. Suppose that $r$, $R$, $\Tth$, and $\mu$ satisfy \footnote{Without loss of generality, here we assume that $\frac{2}{\eta\gamma} \log_{9/4}(\frac{2(R + r)}{\mu})$ is an integer, so that $\Tth$ is an integer.}
\begin{align}
&  \Tth = \frac{2}{\eta\gamma} \log_{9/4}(\frac{2(R + r)}{\mu}), \label{eq:Tth_def} \\
& R \ge \mu, \label{eq:R_mu} \\
&  \rho_F(R + r)\mu \ge \Delta, \label{eq:Delta_R} \\
& \gamma \ge 24 \rho_F(R+r) \log_{9/4} (\frac{2(R + r)}{\mu}). \label{eq:gamma_R_log}
\end{align} 
Then, there must be either $\vecu_0 \notin \mathbb{W}_S$ or $\vecy_0 \notin \mathbb{W}_S$.
\end{lemma}

We prove Lemma~\ref{lem:moving_dist} in Appendix~\ref{prf:moving_dist}. With this lemma, we proceed to analyze the probability that the algorithm escapes the saddle points. In particular, we bound the probability that $\vecw_0 \in \mathbb{W}_S(\tvecw, r, R, \Tth)$ when $\lambda_{\min}(\nabla^2F(\tvecw)) \le -\gamma$ and $\vecw_0$ is drawn from $\Ball_{\tvecw}(r)$ uniform at random.

\begin{lemma}\label{lem:vol_stuck_region}
Assume that $\lambda_{\min}(\nabla^2F(\tvecw)) \le -\gamma < 0$, and let the unit vector $\vece$ be the eigenvector associated with $\lambda_{\min}(\nabla^2F(\tvecw))$. Consider the stuck region $\mathbb{W}_S(\tvecw, r, R, \Tth)$ at $\tvecw$, and suppose that $r$, $R$, $\Tth$, and $\mu$ satisfy the conditions in~\eqref{eq:Tth_def}-\eqref{eq:gamma_R_log}. Then, when we sample $\vecw_0$ from $\Ball_{\tvecw}(r)$ uniformly at random, the probability that $\vecw_0 \in \mathbb{W}_S(\tvecw, r, R, \Tth)$ is at most $\frac{2\mu\sqrt{d}}{r}$.
\end{lemma}

\begin{proof}
Since the starting point $\vecw_0$ is uniformly distributed in $\Ball_{\tvecw}(r)$, to bound the probability of getting stuck, it suffices to bound the volume of $\mathbb{W}_S$. Let $\indi_{\mathbb{W}_S}(\vecw)$ be the indicator function of the set $\mathbb{W}_S$. For any $\vecw\in\R^d$, let $w^{(1)}$ be the projection of $\vecw$ onto the $\vece$ direction, and $\vecw^{(-1)}\in\R^{d-1}$ be the remaining component of $\vecw$. Then, we have
\begin{align*}
\vol(\mathbb{W}_S) = & \int_{ \Ball_{\tvecw}^{(d)}(r) } \indi_{\mathbb{W}_S}(\vecw) \dev \vecw  \\
= & \int_{ \Ball_{\tvecw}^{(d-1)}(r) } \dev \vecw^{(-1)} \int_{\widetilde{w}^{(1)} - \sqrt{r^2 - \twonms{\tvecw^{(-1)} - \vecw^{(-1)}}^2 }}^{\widetilde{w}^{(1)} + \sqrt{r^2 - \twonms{\tvecw^{(-1)} - \vecw^{(-1)}}^2 }}   \indi_{\mathbb{W}_S}(\vecw)  \dev  \widetilde{w}^{(1)}  \\
\le & 2\mu \int_{ \Ball_{\tvecw}^{(d-1)}(r) } \dev \vecw^{(-1)}  \\
= & 2\mu \vol(\Ball_0^{(d-1)}(r)),
\end{align*}
where the inequality is due to Lemma~\ref{lem:moving_dist}. Then, we know that the probability of getting stuck is
\begin{align*}
\frac{ \vol(\mathbb{W}_S) }{ \vol(\Ball_0^{(d)}(r)) } \le & 2\mu \frac{\vol(\Ball_0^{(d-1)}(r))}{\vol(\Ball_0^{(d)}(r))} = \frac{2\mu}{\sqrt{\pi} r} \frac{\Gamma(\frac{d}{2} + 1)}{\Gamma(\frac{d}{2} + \frac{1}{2})} \le \frac{2\mu}{\sqrt{\pi} r} \sqrt{\frac{d}{2} + \frac{1}{2}} \le \frac{2\mu\sqrt{d}}{r},
\end{align*}
where we use the fact that $\frac{\Gamma(x+1)}{\Gamma(x+\frac{1}{2})} < \sqrt{x + \frac{1}{2}}$ for any $x\ge 0$. 
\end{proof}

We then analyze the decrease of value of the population loss function $F(\cdot)$ when we conduct the perturbation step. Assume that we successfully escape the saddle point, i.e., there exists $t \le \Tth$ such that $\twonms{\vecw_t - \vecw_0} \ge R$. The following lemma provides the decrease of $F(\cdot)$.

\begin{lemma}\label{lem:func_value_decay}
Suppose that $\lambda_{\min}(\nabla^2F(\tvecw)) \le -\gamma < 0$, and at $\tvecw$, we observe $\twonms{\vecg(\tvecw)} \le \epsilon=3\Delta$. Assume that $\vecw_0 \in \Ball_{\tvecw}(r)$ and that $\vecw_0 \notin \mathbb{W}_S(\tvecw, r, R, \Tth)$. Let $t \le \Tth$ be the step such that $\twonms{\vecw_t - \vecw_0} \ge R$. Then, we have
\begin{equation}\label{eq:func_decay_main}
F(\tvecw) - F(\vecw_t) \ge \frac{L_F}{4\Tth} R^2 - \frac{ \Delta^2 \Tth}{L_F} -4\Delta r - \frac{L_F}{2}r^2.
\end{equation}
\end{lemma}

We prove Lemma~\ref{lem:func_value_decay} in Appendix~\ref{prf:func_value_decay}.

Next, we choose the quantities $\mu$, $r$, $R$, and $\gamma$ such that (i) the conditions~\eqref{eq:Tth_def}-\eqref{eq:gamma_R_log} in Lemma~\ref{lem:moving_dist} are satisfied, (ii) the probability of escaping saddle point in Lemma~\ref{lem:vol_stuck_region} is at least a constant, and (iii) the decrease in function value in~\eqref{eq:func_decay_main} is large enough. We first choose
\begin{align}
\mu & = \Delta^{3/5}d^{-1/5}\rho_F^{-1/2},  \label{eq:choice_mu} \\
r & = 4 \Delta^{3/5} d^{3/10}\rho_F^{-1/2},  \label{eq:choice_r} \\
R &= \Delta^{2/5} d^{1/5} \rho_F^{-1/2}.  \label{eq:choice_bigR} 
\end{align}
One can simply check that, according to Lemma~\ref{lem:vol_stuck_region}, when we drawn $\vecw_0$ from $\Ball_{\tvecw}(r)$ uniformly at random, the probability that $\vecw_0 \in \mathbb{W}_S(\tvecw, r, R, \Tth)$ is at most $1/2$. Since we assume that $\Delta \le 1$, one can also check that the condition~\eqref{eq:R_mu} is satisfied. In addition, since $\rho_F R \mu = \Delta$, the condition~\eqref{eq:Delta_R} is also satisfied. According to~\eqref{eq:Tth_def}, we have
\begin{equation}\label{eq:Tth_def2}
\Tth = \frac{2L_F}{\gamma} \log_{9/4}(\frac{2d^{2/5}}{\Delta^{1/5}} + 8d^{1/2}).
\end{equation}
In the following, we choose
\begin{equation}\label{eq:choice_gamma}
\gamma = 768 (\rho_F^{1/2} + L_F) (\Delta^{2/5}d^{1/5} + \Delta^{3/5}d^{3/10}) \log_{9/4}(\frac{2d^{2/5}}{\Delta^{1/5}} + 8d^{1/2}),
\end{equation}
which implies
\begin{equation}\label{eq:Tth_def3}
\Tth = \frac{L_F}{384 (\rho_F^{1/2} + L_F) (\Delta^{2/5}d^{1/5} + \Delta^{3/5}d^{3/10}) }
\end{equation}
Then we check condition~\eqref{eq:gamma_R_log} holds. We have
\[
24 \rho_F(R+r) \log_{9/4} (\frac{2(R + r)}{\mu}) =  24 \rho_F^{1/2} (\Delta^{2/5} d^{1/5} + 4\Delta^{3/5}d^{3/10}) \log_{9/4} (\frac{2d^{2/5}}{\Delta^{1/5}} + 8d^{1/2})  \le \gamma.
\]

Next, we consider the decrease in function value in~\eqref{eq:func_decay_main}. Using the equations~\eqref{eq:Tth_def2} and~\eqref{eq:choice_gamma}, we can show the following three inequalities by direct algebra manipulation.
\begin{align}
&\frac{L_F}{4\Tth} R^2 \ge 6 \frac{\Delta^2 \Tth}{L_F},  \label{eq:func_decay_1} \\
&\frac{L_F}{4\Tth} R^2 \ge 24 \Delta r \label{eq:func_decay_2},  \\
&\frac{L_F}{4\Tth} R^2 \ge 3L_Fr^2 \label{eq:func_decay_3}.
\end{align}
By adding up~\eqref{eq:func_decay_1},~\eqref{eq:func_decay_2}, and~\eqref{eq:func_decay_3}, we obtain
\[
\frac{L_F}{4\Tth} R^2 \ge 2 \frac{\Delta^2 \Tth}{L_F} + 8 \Delta r  +  L_Fr^2,
\]
which implies that when we successfully escape the saddle point, we have
\begin{equation}\label{eq:func_decay_4}
F(\tvecw) - F(\vecw_t) \ge \frac{L_F}{8\Tth} R^2 = 48(\rho_F^{-1/2} + L_F\rho_F^{-1}) (\Delta^{6/5}d^{3/5} + \Delta^{7/5}d^{7/10}).
\end{equation}
Then, one can simply check that, the average decrease in function value during the successful round of the $\mathsf{Escape}$ process is
\begin{equation}\label{eq:func_decay_ave}
\frac{F(\tvecw) - F(\vecw_t)}{ t } \ge \frac{F(\tvecw) - F(\vecw_t)}{ \Tth } \ge \frac{2 ( \Delta^{8/5}d^{4/5} + \Delta^2 d)}{L_F} > \frac{3\Delta^2}{2L_F}.
\end{equation}
Recall that according to~\eqref{eq:func_decrease_no_perturb2}, when the algorithm is not in the $\mathsf{Escape}$ process, the function value is decreased by at least $\frac{3\Delta^2}{2L_F}$ in each iteration. Therefore, if the algorithm successfully escapes the saddle point, during the $\mathsf{Escape}$ process, the average decrease in function value is \emph{larger} than the iterations which are not in this process.

So far, we have chosen the algorithm parameters $r$, $R$, $\Tth$, as well as the final second-order convergence guarantee $\gamma$. Now we proceed to analyze the total number of iterations and the failure probability of the algorithm. According to Lemma~\ref{lem:vol_stuck_region} and the choice of $\mu$ and $r$, we know that at each point with $\twonms{\vecg(\tvecw)} \le \epsilon$, the algorithm can successfully escape this saddle point with probability at least $1/2$. To boost the probability of escaping saddle points, we need to repeat the process $Q$ rounds in $\mathsf{Escape}$, independently. Since for each successful round, the function value decrease is at least 
\[
48(\rho_F^{-1/2} + L_F\rho_F^{-1}) (\Delta^{6/5}d^{3/5} + \Delta^{7/5}d^{7/10}) \ge 48 L_F\rho_F^{-1} (\Delta^{6/5}d^{3/5} + \Delta^{7/5}d^{7/10}),
\]
and the function value can decrease at most $F_0 - F^*$. Therefore, the total number of saddle points that we need to escape is at most
\begin{equation}\label{eq:total_saddle}
\frac{\rho_F (F_0 - F^*) }{48 L_F (\Delta^{6/5}d^{3/5} + \Delta^{7/5}d^{7/10}) }.
\end{equation}
Therefore, by union bound, the failure probability of the algorithm is at most
\[
\frac{\rho_F (F_0 - F^*) }{48 L_F (\Delta^{6/5}d^{3/5} + \Delta^{7/5}d^{7/10}) } (\frac{1}{2})^Q,
\]
and to make the failure probability at most $\delta$, one can choose
\begin{equation}\label{eq:choice_Q}
Q \ge 2 \log\left( \frac{ \rho_F( F_0 - F^*) }{48 L_F \delta (\Delta^{6/5}d^{3/5} + \Delta^{7/5}d^{7/10}) } \right).
\end{equation}
Again, due to the fact that the function value decrease is at most $F_0 - F^*$, and in each \emph{effective} iteration, the function value is decreased by at least $\frac{3\Delta^2}{2L_F}$. (Here, the effective iterations are the iterations when the algorithm is not in the $\mathsf{Escape}$ process and the iterations when the algorithm successfully escapes the saddle points.) The total number of effective iterations is at most
\begin{equation}\label{eq:total_effective_iter}
\frac{2(F_0 - F^*)L_F}{3\Delta^2}.
\end{equation}
Combing with~\eqref{eq:choice_Q}, we know that the total number of parallel iterations is at most
\[
\frac{4(F_0 - F^*)L_F}{3\Delta^2} \log\left( \frac{\rho_F (F_0 - F^*) }{48 L_F \delta (\Delta^{6/5}d^{3/5} + \Delta^{7/5}d^{7/10}) } \right).
\]
When all the algorithm terminates, and the saddle point escaping process is successful, the output of the algorithm $\tvecw$ satisfies $\twonms{\vecg(\tvecw)} \le \epsilon$, which implies that $\twonms{\gradF(\tvecw)} \le 4\Delta$, and
\begin{equation}\label{eq:lambda_min_init}
\begin{aligned}
\lambda_{\min}(\hessF(\tvecw)) & \ge -\gamma = -768 (\rho_F^{1/2} + L_F) (\Delta^{2/5}d^{1/5} + \Delta^{3/5}d^{3/10}) \log_{9/4}(\frac{2d^{2/5}}{\Delta^{1/5}} + 8d^{1/2})  \\
&\ge -950 (\rho_F^{1/2} + L_F) (\Delta^{2/5}d^{1/5} + \Delta^{3/5}d^{3/10}) \log(\frac{2d^{2/5}}{\Delta^{1/5}} + 8d^{1/2}).
\end{aligned}
\end{equation}
We next show that we can simplify the guarantee as
\begin{equation}\label{eq:lambda_min_final}
\lambda_{\min}(\hessF(\tvecw)) \ge -1900 (\rho_F^{1/2} + L_F) \Delta^{2/5}d^{1/5} \log(\frac{10}{\Delta}).
\end{equation}
We can see that if $\Delta \le \frac{1}{\sqrt{d}}$, then $ \Delta^{3/5}d^{3/10} \le \Delta^{2/5}d^{1/5}$ and $\frac{2d^{2/5}}{\Delta^{1/5}} + 8d^{1/2} \le \frac{10}{\Delta}$. Thus, the bound~\eqref{eq:lambda_min_final} holds. On the other hand, when $\Delta > \frac{1}{\sqrt{d}}$, we have $\Delta^{2/5}d^{1/5} > 1$ and thus
\[
1900 (\rho_F^{1/2} + L_F) \Delta^{2/5}d^{1/5} \log(\frac{10}{\Delta}) > L_F.
\]
By the smoothness of $F(\cdot)$, we know that $\lambda_{\min}(\hessF(\tvecw)) \ge -L_F$. Therefore, the bound~\eqref{eq:lambda_min_final} still holds, and this completes the proof.

\subsection{Proof of Lemma~\ref{lem:moving_dist}}\label{prf:moving_dist}

We prove by contradiction. Suppose that $\vecu_0,\vecy_0 \in \mathbb{W}_S$.
Let $\{\vecu_t\}$ and $\{\vecy_t\}$ be two sequences generated by the following two iterations:
\begin{align}
\vecu_t &= \vecu_{t-1} - \eta \vecg(\vecu_{t-1}),  \label{eq:seq_ut_stuck} \\
\vecy_t &= \vecy_{t-1} - \eta \vecg(\vecy_{t-1}),  \label{eq:seq_yt_stuck}
\end{align}
respectively, where $\twonms{\vecg(\vecw) - \gradF(\vecw)} \le \Delta$ for any $\vecw \in \W$. According to our assumption, we have $\forall~t \le \Tth$, $\twonms{\vecu_t - \vecu_0} < R$ and $\twonms{\vecy_t - \vecy_0} < R$.

Define $\vecv_t := \vecy_t - \vecu_t$, $\vecdlt_t := \vecg(\vecu_t) - \gradF(\vecu_t)$, and $\vecdlt'_t := \vecg(\vecy_t) - \gradF(\vecy_t)$. Then we have
\begin{align*}
\vecy_{t+1} & = \vecy_t - \eta(\gradF(\vecy_t) + \vecdlt'_t) \\
& = \vecu_t + \vecv_t - \eta(\gradF(\vecu_t + \vecv_t) + \vecdlt'_t)  \\
& = \vecu_t + \vecv_t - \eta\gradF(\vecu_t) - \eta \left[ \int_0^1\hessF(\vecu_t + \theta\vecv_t) \right]\vecv_t - \eta\vecdlt'_t  \\
& = \vecu_{t+1} + \eta\vecdlt_t + \vecv_t  - \eta \left[ \int_0^1 \hessF(\vecu_t + \theta\vecv_t) \dev \theta \right]\vecv_t - \eta\vecdlt'_t,
\end{align*}
which yields
\begin{equation}\label{eq:iter_vt}
\vecv_{t+1} = (\matI - \eta\matH) \vecv_t - \eta\matQ_t \vecv_t + \eta(\vecdlt_t - \vecdlt'_t),
\end{equation}
where
\begin{equation}\label{eq:def_qtprimte}
\matQ_t := \int_0^1 \hessF(\vecu_t + \theta\vecv_t) \dev \theta - \matH.
\end{equation}
By the Hessian Lipschitz property, we know that
\begin{equation}\label{eq:op_norm_qtprime2}
\begin{aligned}
\twonms{\matQ_t} \le & \rho_F(\twonms{\vecu_t - \tvecw} + \twonms{\vecy_t - \tvecw} ) \\
\le & \rho_F(\twonms{\vecu_t - \vecu_0} + \twonms{\vecu_0 - \tvecw} + \twonms{\vecy_t - \vecy_0} + \twonms{\vecy_0 - \tvecw} ) \\
\le & 2\rho_F(R+r).
\end{aligned}
\end{equation}
We let $\psi_t$ be the norm of the projection of $\vecv_t$ onto the $\vece$ direction, and $\phi_t$ be the norm of the projection of $\vecv_t$ onto the remaining subspace. By definition, we have $\psi_0 = \mu_0 \ge \mu > 0$ and $\phi_0 = 0$. According to~\eqref{eq:iter_vt} and~\eqref{eq:op_norm_qtprime2}, we have
\begin{align}
\psi_{t+1} & \ge (1+\eta\gamma)\psi_t - 2 \eta \rho_F(R+r) \sqrt{\psi_t^2 + \phi_t^2} - 2\eta\Delta,  \label{eq:iter_psi} \\
\phi_{t+1} & \le  (1+\eta\gamma)\phi_t + 2 \eta \rho_F(R+r) \sqrt{\psi_t^2 + \phi_t^2} + 2\eta \Delta.  \label{eq:iter_phi}
\end{align}

In the following, we use induction to prove that $\forall~t \le \Tth$, 
\begin{equation}\label{eq:induction_arg}
\psi_t \ge (1+\frac{1}{2}\eta\gamma)\psi_{t-1} \quad \text{and} \quad \phi_t \le \frac{t}{\Tth}\psi_t   
\end{equation}
We know that~\eqref{eq:induction_arg} holds when $t=0$ since we have $\phi_0 = 0$. Then, assume that for some $t < \Tth$, we have $\forall~\tau \le t$, $\psi_\tau \ge (1+\frac{1}{2}\eta\gamma) \psi_{\tau-1}$ and $\phi_\tau \le \frac{\tau}{\Tth} \psi_\tau$. We show that~\eqref{eq:induction_arg} holds for $t+1$.

First, we show that $\psi_{t+1} \ge (1+\frac{1}{2}\eta\gamma) \psi_{t}$. Since $\forall~\tau \le t$, $\psi_\tau \ge \psi_{\tau-1}$, we know that $\psi_{t} \ge \psi_0  \ge \mu$. Therefore, according to~\eqref{eq:Delta_R}, we have 
\begin{equation}\label{eq:delta_up_psi}
\Delta \le \rho_F(R+r) \mu \le \rho_F(R+r) \psi_t.
\end{equation}
In addition, since $t < \Tth$, we have 
\begin{equation}\label{eq:phi_psi}
\phi_t \le \psi_t.
\end{equation}
Combining~\eqref{eq:delta_up_psi},~\eqref{eq:phi_psi} and~\eqref{eq:iter_psi},~\eqref{eq:iter_phi}, we get
\begin{align}
\psi_{t+1} & \ge (1+\eta\gamma)\psi_t - 2 \eta \rho_F(R+r) \sqrt{2\psi_t^2} - 2\eta \rho_F(R+r) \psi_t > (1+\eta\gamma)\psi_t - 6\eta \rho_F(R+r) \psi_t,  \label{eq:iter_psi2} \\
\phi_{t+1} & \le  (1+\eta\gamma)\phi_t + 2 \eta \rho_F(R+r) \sqrt{2\psi_t^2} + 2\eta \rho_F(R+r) \psi_t <  (1+\eta\gamma)\phi_t + 6\eta \rho_F(R+r) \psi_t.  \label{eq:iter_phi2}
\end{align}
According to~\eqref{eq:gamma_R_log}, we have $\gamma \ge 24 \rho_F(R+r) \log_{9/4} (\frac{2(R + r)}{\mu}) > 12\rho_F(R+r)$. Combining with~\eqref{eq:iter_psi2}, we know that $\psi_{t+1} \ge (1+ \frac{1}{2}\eta\gamma)\psi_t $.

Next, we show that $\phi_{t+1} \le \frac{t+1}{\Tth}\psi_{t+1}$. Combining with~\eqref{eq:iter_psi2} and~\eqref{eq:iter_phi2}, we know that to show $\phi_{t+1} \le \frac{t+1}{\Tth}\psi_{t+1}$, it suffices to show
\begin{equation}\label{eq:suffice1}
(1+\eta \gamma)\phi_t + 6\eta\rho_F(R+r)\psi_t \le \frac{t+1}{\Tth} [1+\eta\gamma - 6\eta\rho_F(R+r)] \psi_t.
\end{equation}
According to the induction assumption, we have $\phi_{t} \le \frac{t}{\Tth}\psi_{t}$. Then, to show~\eqref{eq:suffice1}, it suffices to show that
\begin{equation}\label{eq:suffice2}
(1+\eta \gamma)t + 6\eta\rho_F(R+r)\Tth \le (t+1)  [1+\eta\gamma - 6\eta\rho_F(R+r)] 
\end{equation}
Since $t+1\le \Tth$, we know that to show~\eqref{eq:suffice2}, it suffices to show
\begin{equation}\label{eq:suffice3}
12 \eta\rho_F(R+r)\Tth \le 1.
\end{equation}
Then, according to~\eqref{eq:Tth_def} and~\eqref{eq:gamma_R_log}, we know that~\eqref{eq:suffice3} holds, which completes the induction.

Next, according to~\eqref{eq:induction_arg}, we know that
\begin{align*}
\twonms{\vecu_{\Tth} - \vecy_{\Tth}} &\ge  \phi_{\Tth} \ge (1+\frac{1}{2} \eta \gamma)^{\Tth}\mu_0  \\
& \ge (1+\frac{1}{2} \eta \gamma)^{ \frac{2}{\eta\gamma} \log_{9/4}(\frac{2(R+r)}{\mu}) }\mu_0  \\
& \ge \frac{2(R+r)}{\mu} \cdot \mu_0 = 2(R+r),
\end{align*}
where the last inequality is due to the fact that $\eta = \frac{1}{L_F}$ and thus $\eta \gamma \le 1$. On the other hand, since we assume that $\vecu_0,\vecy_0 \in \mathbb{W}_S$, we know that
\[
\twonms{\vecu_{\Tth} - \vecy_{\Tth}} \le \twonms{\vecu_{\Tth} - \vecu_0 } + \twonms{\vecy_{\Tth} - \vecy_0 } + \twonms{\vecu_0 - \vecy_0}  < 2(R +r),
\]
which leads to contradiction and thus completes the proof.

\subsection{Proof of Lemma~\ref{lem:func_value_decay}}\label{prf:func_value_decay}

Recall that we have the iterations $\vecw_{\tau+1} = \vecw_\tau - \eta \vecg(\vecw_\tau)$ for all $\tau < t$. Let $\vecdlt_\tau = \gradF(\vecw_{\tau}) - \vecg(\vecw_\tau)$, and then $\twonms{\vecdlt_\tau} \le \Delta$. By the smoothness of $F(\cdot)$ and the fact that $\eta = \frac{1}{L_F}$, we have
\begin{equation}\label{eq:smoothness_decay}
\begin{aligned}
F(\vecw_\tau ) - F(\vecw_{\tau+1}) \ge & \innerps{\gradF(\vecw_\tau)}{\vecw_\tau - \vecw_{\tau+1}} - \frac{L_F}{2}\twonms{\vecw_{\tau} - \vecw_{\tau+1}}^2  \\
= & \innerp{\frac{\vecw_\tau - \vecw_{\tau+1}}{\eta} + \vecdlt_\tau}{\vecw_\tau - \vecw_{\tau+1}} - \frac{L_F}{2}\twonms{\vecw_\tau - \vecw_{\tau+1}}^2 \\
= & \frac{L_F}{2}\twonms{\vecw_\tau - \vecw_{\tau+1}}^2 + \innerps{\vecdlt_\tau}{\vecw_\tau - \vecw_{\tau+1}} \\
\ge & \frac{L_F}{4} \twonms{\vecw_\tau - \vecw_{\tau+1}}^2 - \frac{ \twonms{\vecdlt_\tau}^2 }{L_F} \\
\ge & \frac{L_F}{4} \twonms{\vecw_\tau - \vecw_{\tau+1}}^2 - \frac{ \Delta^2 }{L_F}.
\end{aligned}
\end{equation}
By summing up~\eqref{eq:smoothness_decay} for $\tau = 0,1,\ldots, t-1$, we get
\begin{equation}\label{eq:decay_w0_wt}
F(\vecw_0 ) - F(\vecw_t) \ge \frac{L_F}{4} \sum_{\tau=0}^{t-1} \twonms{\vecw_\tau - \vecw_{\tau+1}}^2 - \frac{ \Delta^2 t}{L_F}.
\end{equation}
Consider the $k$-th coordinate of $\vecw_\tau$ and $\vecw_{\tau+1}$, by Cauchy-Schwarz inequality, we have
\[
\sum_{\tau=0}^{t-1} (w_{\tau, k} - w_{\tau+1, k})^2 \ge \frac{1}{t} (w_{0,k} - w_{t,k})^2,
\]
which implies
\begin{equation}\label{eq:cauthy_dist}
\sum_{\tau=0}^{t-1} \twonms{\vecw_\tau - \vecw_{\tau+1}}^2 \ge \frac{1}{t} \twonms{\vecw_0 - \vecw_t}^2.
\end{equation}
Combining~\eqref{eq:decay_w0_wt} and~\eqref{eq:cauthy_dist}, we obtain
\begin{equation}\label{eq:decay_w0_wt2}
F(\vecw_0 ) - F(\vecw_t) \ge \frac{L_F}{4t} \twonms{\vecw_0 - \vecw_t}^2 - \frac{ \Delta^2 t}{L_F} \ge  \frac{L_F}{4\Tth} R^2 - \frac{ \Delta^2 \Tth}{L_F}.
\end{equation}
On the other hand, by the smoothness of $F(\cdot)$, we have
\begin{equation}\label{eq:tw_w0}
F(\tvecw) - F(\vecw_0) \ge \innerps{\gradF(\tvecw)}{\tvecw - \vecw_0} - \frac{L_F}{2}\twonms{\vecw_0 - \tvecw}^2 \ge -(\epsilon+\Delta)r - \frac{L_F}{2}r^2.
\end{equation}
Adding up~\eqref{eq:decay_w0_wt2} and~\eqref{eq:tw_w0}, we obtain
\begin{equation}\label{eq:function_value_inc}
F(\tvecw) - F(\vecw_t) \ge \frac{L_F}{4\Tth} R^2 - \frac{ \Delta^2 \Tth}{L_F} -(\epsilon+\Delta)r - \frac{L_F}{2}r^2,
\end{equation}
which completes the proof.

\section{Proof of Theorem~\ref{thm:exact_oracle}}\label{apx:exact_oracle}

First, when we run gradient descent iterations $\vecw' = \vecw - \eta\gradF(\vecw)$, according to Lemma~\ref{lem:each_iter}, we have
\begin{equation}\label{eq:exact_iter}
F(\vecw') \le F(\vecw) - \frac{1}{2L_F}\twonms{\gradF(\vecw)}^2.
\end{equation}
Suppose at $\tvecw$, we observe that $\twonms{\gradF(\tvecw)} \le \epsilon$, and then we start the $\mathsf{Escape}$ process. When we have exact gradient oracle, we can still define the stuck region $\mathbb{W}_S$ at $\tvecw$ as in the definition of stuck region in Appendix~\ref{prf:main}, by simply replacing the inexact gradient oracle with the exact oracle. Then, we can analyze the size of the stuck region according to Lemma~\ref{lem:moving_dist}. Assume that the smallest eigenvalue of $\matH := \nabla^2F(\tvecw)$ satisfies $\lambda_{\min}(\matH) \le -\gamma < 0$,
and let the unit vector $\vece$ be the eigenvector associated with $\lambda_{\min}(\matH)$. Let $\vecu_0, \vecy_0 \in \Ball_{\tvecw}(r)$ be two points such that $\vecy_0 = \vecu_0 + \mu_0\vece$ with some $\mu_0 \ge \mu \in (0, r)$. Consider the stuck region $\mathbb{W}_S(\tvecw, r, R, \Tth)$. Suppose that $r$, $R$, $\Tth$, and $\mu$ satisfy 
\begin{align}
&  \Tth = \frac{2}{\eta\gamma} \log_{9/4}(\frac{2(R + r)}{\mu}), \label{eq:exact_Tth_def} \\
& R \ge \mu, \label{eq:exact_R_mu} \\
& \gamma \ge 24 \rho_F(R+r) \log_{9/4} (\frac{2(R + r)}{\mu}). \label{eq:exact_gamma_R_log}
\end{align} 
Then, there must be either $\vecu_0 \notin \mathbb{W}_S$ or $\vecy_0 \notin \mathbb{W}_S$. In addition, according to Lemma~\ref{lem:vol_stuck_region}, if conditions~\eqref{eq:exact_Tth_def}-\eqref{eq:exact_gamma_R_log} are satisfied, then, when we sample $\vecw_0$ from $\Ball_{\tvecw}(r)$ uniformly at random, the probability that $\vecw_0 \in \mathbb{W}_S(\tvecw, r, R, \Tth)$ is at most $\frac{2\mu\sqrt{d}}{r}$. In addition, according to~\eqref{eq:function_value_inc} in the proof of Lemma~\ref{lem:func_value_decay}, assume that $\vecw_0 \in \Ball_{\tvecw}(r)$ and that $\vecw_0 \notin \mathbb{W}_S(\tvecw, r, R, \Tth)$. Let $t \le \Tth$ be the step such that $\twonms{\vecw_t - \vecw_0} \ge R$. Then, we have
\begin{equation}\label{eq:exact_func_decay_main}
F(\tvecw) - F(\vecw_t) \ge \frac{L_F}{4\Tth} R^2  - \epsilon r - \frac{L_F}{2}r^2.
\end{equation}
Combining~\eqref{eq:exact_Tth_def} and~\eqref{eq:exact_gamma_R_log}, we know that the first term on the right hand side of~\eqref{eq:exact_func_decay_main} satisfies
\begin{equation}\label{eq:exact_func_first}
 \frac{L_F}{4\Tth} R^2 \ge 3\rho_F R^3.
\end{equation}
Choose $R = \sqrt{\epsilon/\rho_F}$ and $r = \epsilon$. Then, we know that when $\epsilon \le \min\{ \frac{1}{\rho_F}, \frac{4}{L_F^2\rho_F} \}$, we have $\epsilon r \le \rho_F R^3$ and $\frac{1}{2} L_F r^2 \le \rho_FR^3$. Combining these facts with~\eqref{eq:exact_func_decay_main} and~\eqref{eq:exact_func_first}, we know that, when the algorithm successfully escapes the saddle point, the decrease in function value satisfies
\begin{equation}\label{eq:exact_func_dec}
F(\tvecw) - F(\vecw_t) \ge \rho_F R^3.
\end{equation}
Therefore, the average function value decrease during the $\mathsf{Escape}$ process is at least
\begin{equation}\label{eq:exact_func_dec_ave}
\frac{F(\tvecw) - F(\vecw_t) }{\Tth} \ge \frac{12}{L_F} \epsilon^2.
\end{equation}
When we have exact gradient oracle, we choose $Q=1$. According to~\eqref{eq:exact_iter} and~\eqref{eq:exact_func_dec_ave}, for the iterations that are not in the $\mathsf{Escape}$ process, the function value decrease in each iteration is at least $\frac{1}{2L_F}\epsilon^2$; for the iterations in the $\mathsf{Escape}$ process, the function value decrease on average is $\frac{12}{L_F} \epsilon^2$. Since the function value can decrease at most $F_0 - F^*$, the algorithm must terminate within $\frac{2L_F(F_0 - F^*)}{\epsilon^2}$ iterations.

The we proceed to analyze the failure probability. We can see that the number of saddle points that the algorithm may need to escape is at most $\frac{F_0 - F^*}{\rho_F R^3} $. Then, by union bound the probability that the algorithm fails to escape one of the saddle points is at most
\[
\frac{2\mu \sqrt{d}}{r} \cdot \frac{F_0 - F^*}{\rho_F R^3}
\]
By letting the above probability to be $\delta$, we obtain
\[
\mu = \frac{\delta \epsilon^{5/2}}{2\sqrt{\rho_F d}(F_0 - F^*)},
\]
which completes the proof.

\section{Proof of Proposition~\ref{obs:second_order_lb}}\label{prf:second_order_lb}
We consider the following class of one-dimensional functions indexed by $ s\in \R $: 
\[
\mathcal{F} = \{f_s(\cdot): f_s(w) = \Delta^{3/2}\sin (\Delta^{-1/2}w + s), s\in\R\}.
\]
Then, for each function $f_s(\cdot)\in\mathcal{F}$, we have 
\[
\gradf_s(w) = \Delta \cos (\Delta^{-1/2}w + s),
\]
and 
\[
\hessf_s (w) = -\Delta^{1/2} \sin(\Delta^{-1/2}w + s).
\]
Thus, we always have $| \gradf_s(w) | \le \Delta, \forall w$. Therefore, the $\Delta$-inexact gradient oracle can simply output $0$ all the time. In addition, we verify that for all $ s $ and $w $, $ |\hessf_s(w) | \le \Delta^{1/2} \le 1$ and $ |\nabla^3 f_s(w) | =| -\cos (\Delta^{-1/2}w + s) | \le 1 $ under the assumption that $ \Delta \le 1 $, so all the functions in $\mathcal{F}$ are $1$-smooth and $1$-Hessian Lipschitz as claimed.

In this case, the output of the algorithm does not depend on $ s $, that is, the actual function that we aim to minimize. Consequently,  for any output $\widetilde{w}$ of the algorithm, there exists $s\in\R$ such that $ \Delta^{-1/2}\widetilde{w} + s = \pi/4$, and thus $ | \gradf_s(\widetilde{w}) | =\Delta /\sqrt{2} $ and $\lambda_{\min}(\hessf_s(\widetilde{w})) = -\Delta^{1/2} /\sqrt{2}$.

\section{Proof of Proposition~\ref{ppn:size_w_space}}\label{prf:size_w_space}
Suppose that during all the iterations, the $\mathsf{Escape}$ process is called $E+1$ times. In the first $E$ times, the algorithm escapes the saddle points, and in the last $\mathsf{Escape}$ process, the algorithm does not escape and outputs $\tvecw$. For the first $E$ processes, there might be up to $Q$ rounds of perturb-and-descent operations, and we only consider the successful descent round. We can then partition the algorithm into $E+1$ segments. We denote the starting and ending iterates of the $t$-th segment by $\vecw_t$ and $\tvecw_t$, respectively, and denote the length (number of inexact gradient descent iterations) by $T_t$. When the algorithm reaches $\tvecw_t$, we randomly perturb $\tvecw_t$ to $\vecw_{t+1}$, and thus we have $\twonms{\tvecw_t - \vecw_{t+1}} \le r$ for every $t = 0,1,\ldots, E-1$. According to~\eqref{eq:total_effective_iter}, we know that
\[
\sum_{t=0}^E T_t \le \frac{2(F_0 - F^*)L_F}{3\Delta^2} := \widetilde{T},
\]
and according to~\eqref{eq:total_saddle}, we have
\[
E \le \frac{\rho_F (F_0 - F^*) }{48 L_F (\Delta^{6/5}d^{3/5} + \Delta^{7/5}d^{7/10}) }.
\]

According to~\eqref{eq:decay_w0_wt2}, we know that
\[
F(\vecw_t) - F(\tvecw_t) \ge \frac{L_F}{4T_t} \twonms{\vecw_t - \tvecw_t}^2 - \frac{\Delta^2T_t}{L_F},
\]
which implies
\[
\twonms{\vecw_t - \tvecw_t} \le \frac{2}{\sqrt{L_F}} \sqrt{T_t(F(\vecw_t) - F(\tvecw_t))} + \frac{2\Delta T_t}{L_F}.
\]
Then, by Cauchy-Schwarz inequality, we have
\begin{equation}\label{eq:total_dist_1}
\sum_{t=0}^E \twonms{\vecw_t - \tvecw_t} \le 2 \sqrt{\frac{\widetilde{T}}{L_F} \sum_{t=0}^E (F(\vecw_t) - F(\tvecw_t))} + \frac{2\Delta\widetilde{T}}{L_F}.
\end{equation}
On the other hand, we have
\[
\sum_{t=0}^E ( F(\vecw_t) - F(\tvecw_t) ) + \sum_{t=0}^{E-1} ( F(\tvecw_t) - F(\vecw_{t+1}) ) = F(\vecw_0) - F(\tvecw_E) \le F(\vecw_0) - F^*.
\]
According to~\eqref{eq:tw_w0}, we have
\[
F(\tvecw_t) - F(\vecw_{t+1}) \ge -4\Delta r - \frac{L_F}{2}r^2,
\]
and thus
\begin{equation}\label{eq:total_dist_2}
\sum_{t=0}^E ( F(\vecw_t) - F(\tvecw_t) ) \le F(\vecw_0) - F^* + E(4\Delta r + \frac{L_F}{2}r^2)
\end{equation}
Combining~\eqref{eq:total_dist_1} and~\eqref{eq:total_dist_2}, and using the bounds for $\widetilde{T}$ and $E$, we obtain that
\begin{equation}\label{eq:total_dist_3}
\sum_{t=0}^E \twonms{\vecw_t - \tvecw_t} \le C_1 \frac{F(\vecw_0) - F^*}{\Delta},
\end{equation}
where $C_1>0$ is a quantity that only depends on $L_F$ and $\rho_F$.
In addition, we have 
\begin{equation}\label{eq:total_dist_4}
\sum_{t=0}^{E-1} \twonms{\tvecw_t - \vecw_{t+1}} \le Er \le C_2 \frac{F(\vecw_0) - F^*}{\Delta^{3/5}d^{3/10} + \Delta^{4/5}d^{2/5}},
\end{equation}
where $C_2>0$ is a quantity that only depends on $L_F$ and $\rho_F$. 
Combining~\eqref{eq:total_dist_3} and~\eqref{eq:total_dist_4}, and using triangle inequality, we know that
\[
\twonms{\tvecw_{E} - \vecw_0} \le C_1 \frac{F(\vecw_0) - F^*}{\Delta} + C_2 \frac{F(\vecw_0) - F^*}{\Delta^{3/5}d^{3/10} + \Delta^{4/5}d^{2/5}} \le C \frac{F(\vecw_0) - F^*}{\Delta}.
\]
Here, the last inequality is due to the fact that we consider the regime where $\Delta\rightarrow 0$, and $C$ is a quantity that only depends on $L_F$ and $\rho_F$. As a final note, the analysis above also applies to any iterate prior to the final output, and thus, all the iterates during the algorithm stays in the $\ell_2$ ball centered at $\vecw_0$ with radius $C \frac{F(\vecw_0) - F^*}{\Delta}$.

\section{Robust Estimation of Gradients}\label{apx:robust_estimation}

\subsection{Iterative Filtering Algorithm}\label{apx:filtering}

We describe an iterative filtering algorithm for robust mean estimation. The algorithm is originally proposed for robust mean estimation for Gaussian distribution in~\cite{diakonikolas2016robust}, and extended to sub-Gaussian distribution in~\cite{diakonikolas2017being}; then algorithm is reinterpreted in~\cite{steinhardt2017resilience}. Here, we present the algorithm using the interpretation in~\cite{steinhardt2017resilience}. Suppose that $ m $ random vectors $\vecx_1, \vecx_2, \ldots, \vecx_m \in \R^d$ are drawn i.i.d.\ from some distribution with mean~$\vecmu$. An adversary observes all these vectors and changes an $\alpha$ fraction of them in an arbitrary fashion, and we only have access to the corrupted data points $\hvecx_1, \hvecx_2, \ldots, \hvecx_m$. The goal of the iterative filtering algorithm is to output an accurate estimate of the true mean $\vecmu$ even when the dimension~$d$ is large. We provide the detailed procedure in Algorithm~\ref{alg:iter_filtering}. Here, we note that the algorithm parameter $\sigma$ needs to be chosen properly in order to achieve the best possible statistical error rate.

\begin{algorithm}[h]
  \caption{Iterative Filtering~\cite{diakonikolas2016robust,diakonikolas2017being,steinhardt2017resilience}}
  \begin{algorithmic}
  \REQUIRE corrupted data $\hvecx_1, \hvecx_2, \ldots, \hvecx_m\in\R^d$, $\alpha \in [0,\frac{1}{4})$, and algorithm parameter $\sigma > 0$.
  \STATE $\A \leftarrow [m]$, $c_i \leftarrow 1$, and $\tau_i \leftarrow 0$, $\forall~i\in\A$.
   \WHILE{ true }
    \STATE Let $\mat{W}\in\R^{|\A| \times |\A|}$ be a minimizer of the convex optimization problem:
     \STATE
     \[
     \min_{\substack{0 \le W_{ji}\le \frac{3+\alpha}{(1-\alpha)(3-\alpha)m} \\ \sum_{j\in\A} W_{ji}=1}} \max_{\substack{ \matU \succeq 0 \\  \tr(\matU) \le 1 } }  \sum_{i\in\A} c_i (\hvecx_i - \sum_{j\in\A}\hvecx_jW_{ji})^\top \matU (\hvecx_i - \sum_{j\in\A}\hvecx_jW_{ji}),
     \]
     \STATE and $\matU\in\R^{d\times d}$ be a maximizer of the convex optimization problem:
     \[
      \max_{\substack{ \matU \succeq 0 \\  \tr(\matU) \le 1 } }  \min_{\substack{0 \le W_{ji}\le \frac{3+\alpha}{(1-\alpha)(3-\alpha)m} \\ \sum_{j\in\A} W_{ji}=1}} \sum_{i\in\A} c_i (\hvecx_i - \sum_{j\in\A}\hvecx_jW_{ji})^\top \matU (\hvecx_i - \sum_{j\in\A}\hvecx_jW_{ji}).
     \]
     \STATE $\forall~i\in\A,~\tau_i \leftarrow (\hvecx_i - \sum_{j\in\A}\hvecx_jW_{ji})^\top \matU (\hvecx_i - \sum_{j\in\A}\hvecx_jW_{ji})$.
      \IF{$ \sum_{i\in\A} c_i\tau_i > 8m\sigma^2 $} 
      \STATE $\forall~i\in\A,~c_i\leftarrow (1-\frac{\tau_i}{\tau_{\max}})c_i$, where $\tau_{\max} = \max_{i\in\A}\tau_i$.
      \STATE $\A \leftarrow \A \setminus \{i:c_i \le \frac{1}{2}\}$.
      \ELSE
      \STATE \textbf{return} $\widehat{\vecmu} = \frac{1}{|\A|} \sum_{i\in\A}\hvecx_i$      
      \ENDIF
   \ENDWHILE
  \end{algorithmic}\label{alg:iter_filtering}
\end{algorithm} 

\subsection{Proof of Theorem~\ref{thm:iterative_filtering}}\label{prf:iterative_filtering}
To prove Theorem~\ref{thm:iterative_filtering}, we first state a result that bounds the error of the iterative filtering algorithm when the original data points $\{ \vecx_i \} $ are deterministic. The following lemma is proved in~\cite{diakonikolas2017being,steinhardt2017resilience}; also see~\cite{su2018securing} for additional discussion.

\begin{lemma}\label{lem:deterministic}
\cite{diakonikolas2017being,steinhardt2017resilience} Let $\setS:=\{\vecx_1, \vecx_2, \ldots, \vecx_m\}$ be the set of original data points and $\vecmu_\setS:=\frac{1}{m} \sum_{i=1}^m \vecx_i$ be their sample mean. Let $\hvecx_1, \hvecx_2, \ldots, \hvecx_m$ be the corrupted data. If $\alpha \le \frac{1}{4}$, and the algorithm parameter $\nu $ is chosen such that
\begin{equation}\label{eq:condi_sigma}
\twonm{ \frac{1}{m}\sum_{i=1}^m (\vecx_i - \vecmu_\setS)(\vecx - \vecmu_\setS)^\top } \le \nu^2,
\end{equation}
then the output of the iterative filtering algorithm satisfies $\twonms{\widehat{\vecmu} - \vecmu_\setS} \le \bigo(\nu\sqrt{\alpha})$.
\end{lemma} 

By triangle inequality, we have
\begin{equation}\label{eq:mean_triangle}
\twonms{\widehat{\vecmu} - \vecmu} \le \twonms{\widehat{\vecmu} - \vecmu_{\setS}} + \twonms{\vecmu_\setS - \vecmu },
\end{equation}
and
\begin{align}
\twonm{ \frac{1}{m}\sum_{i=1}^m (\vecx_i - \vecmu_\setS)(\vecx - \vecmu_\setS)^\top } = & \frac{1}{m} \twonm{ ([\vecx_1,\cdots,\vecx_m] - \vecmu_\setS \vect{1}^\top) ([\vecx_1,\cdots,\vecx_m] - \vecmu_\setS \vect{1}^\top)^\top}  \nonumber \\
=& \frac{1}{m}\twonm{ [\vecx_1,\cdots,\vecx_m] - \vecmu_\setS \vect{1}^\top }^2   \nonumber \\
\le  & \frac{1}{m} \Big( \twonms{[\vecx_1,\cdots,\vecx_m] - \vecmu \vect{1}^\top} + \sqrt{m} \twonms{\vecmu - \vecmu_\setS} \Big)^2, \label{eq:opnorm_triangle}  
\end{align}
where $\vect{1}$ denotes the all-one vector.\footnote{We note that similar derivation also appears in~\cite{su2018securing}.} 
By choosing 
\[
\nu = \Theta(\frac{1}{\sqrt{m}} \twonms{[\vecx_1,\cdots,\vecx_m] - \vecmu \vect{1}^\top} + \twonms{\vecmu - \vecmu_\setS} )
\]
in Lemma~\ref{lem:deterministic} and combining with the bounds~\eqref{eq:mean_triangle} and~\eqref{eq:opnorm_triangle}, we obtain that
\begin{equation}\label{eq:sufficient}
\twonms{\widehat{\vecmu} - \vecmu} \lesssim \frac{\sqrt{\alpha}}{\sqrt{m}}\twonms{[\vecx_1,\cdots,\vecx_m] - \vecmu \vect{1}^\top} + \twonms{\vecmu - \vecmu_\setS}. 
\end{equation}

With the above bound in hand, we now turn to the robust gradient estimation problem, where the data points are drawn i.i.d. from some unknown distribution. Let $\vecg(\vecw):=\mathsf{filter}\{\vecg_i(\vecw)\}_{i=1}^m$, where $\mathsf{filter}$ represents the iterative filtering algorithm.
In light of~\eqref{eq:sufficient}, we know that in order to bound the gradient estimation error $\sup_{\vecw\in\W} \twonms{\vecg(\vecw) - \gradF(\vecw)}$, it suffices to bound the quantities
\[
\sup_{\vecw\in\W} \twonms{[\gradF_1(\vecw),\cdots,\gradF_m(\vecw)] - \gradF(\vecw) \vect{1}^\top}
\]
and 
\[
\sup_{\vecw\in\W}\twonms{\frac{1}{m}\sum_{i=1}^m\gradF_i(\vecw) - \gradF(\vecw)}.
\]
Here, we recall that $\gradF_i(\vecw)$ is the true gradient of the empirical loss function on the $i$-th machine, and $\vecg_i(\vecw)$ is the (possibly) corrupted gradient.

We first bound $\sup_{\vecw\in\W}\twonms{\frac{1}{m}\sum_{i=1}^m\gradF_i(\vecw) - \gradF(\vecw)}$. Note that we have $\frac{1}{m}\sum_{i=1}^m\gradF_i(\vecw) = \frac{1}{nm} \sum_{i=1}^m\sum_{j=1}^n \gradf(\vecw;\vecz_{i,j})$.
Using the same method as in the proof of Lemma 6 in~\cite{chen2017distributed}, we can show that for each fixed $\vecw$, with probability at least $1-\delta$,
\[
\twonms{\frac{1}{m}\sum_{i=1}^m\gradF_i(\vecw) - \gradF(\vecw)} \le \frac{2\sqrt{2} \zeta}{\sqrt{nm}} \sqrt{d\log 6 + \log\Big(\frac{1}{\delta}\Big)}.
\]
For some $\delta_0 > 0$ to be chosen later, let $\W_{\delta_0} = \{\vecw^1,\vecw^2,\ldots,\vecw^{N_{\delta_0}}\}$ be a finite subset of $\W$ such that for any $\vecw\in\W$, there exists some $\vecw^\ell \in \W_{\delta_0}$ such that $\twonms{\vecw^\ell - \vecw} \le \delta_0$. Standard $ \epsilon $-net results from~\cite{vershynin2010introduction} ensure that $N_{\delta_0} \le (1+\frac{D}{\delta_0})^d$. Then, by the union bound, we have with probability $1-\delta$, for all $\vecw^\ell \in \W_{\delta_0}$,
\begin{equation}\label{eq:in_delta_net}
\twonms{\frac{1}{m}\sum_{i=1}^m\gradF_i(\vecw^\ell) - \gradF(\vecw^\ell)} \le \frac{2\sqrt{2} \zeta}{\sqrt{nm}} \sqrt{d\log 6 + \log\Big(\frac{N_{\delta_0}}{\delta}\Big)}.
\end{equation}
When~\eqref{eq:in_delta_net} holds, by the smoothness of $f(\cdot;\vecz)$ we know that for all $\vecw\in\W$,
\[
\twonms{\frac{1}{m}\sum_{i=1}^m\gradF_i(\vecw) - \gradF(\vecw)} \le \frac{2\sqrt{2} \zeta}{\sqrt{nm}} \sqrt{d\log 6 + \log\Big(\frac{N_{\delta_0}}{\delta}\Big)} + 2L\delta_0.
\]
By choosing $\delta_0 = \frac{1}{nmL}$ and $\delta = \frac{1}{(1+mnDL)^d}$, we obtain that with probability at least $1-\frac{1}{(1+mnDL)^d}$, for all $\vecw\in\W$,
\begin{equation}\label{eq:unif_grad_norm}
\twonms{\frac{1}{m}\sum_{i=1}^m\gradF_i(\vecw) - \gradF(\vecw)} \lesssim \frac{\zeta}{\sqrt{nm}}\sqrt{d\log(1+nmDL)}.
\end{equation}

We next bound $\sup_{\vecw\in\W} \twonms{[\gradF_1(\vecw),\cdots,\gradF_m(\vecw)] - \gradF(\vecw) \vect{1}^\top}$. We note that when the gradients are sub-Gaussian distributed, similar results for the centralized setting have been established in~\cite{charikar2017learning}. One can check that for every $i$, $\gradF_i(\vecw) - \gradF(\vecw)$ is $\frac{\zeta}{\sqrt{n}}$-sub-Gaussian. 
Define $\matG(\vecw) := [\gradF_1(\vecw),\cdots,\gradF_m(\vecw)] - \gradF(\vecw) \vect{1}^\top$.
Using a standard concentration inequality for the norm of a matrix with independent sub-Gaussian columns~\cite{vershynin2010introduction}, we obtain that for each fixed $\vecw$, with probability at least $1-\delta$,
\begin{equation}\label{eq:subgaussian}
\twonms{ \frac{1}{m}\matG(\vecw)\matG(\vecw)^\top - \frac{1}{n}\mat{\Sigma}(\vecw) } \lesssim \frac{\zeta^2}{n}\left( \sqrt{\frac{d}{m}} + \frac{d}{m} + \frac{1}{m}\log\Big(\frac{1}{\delta}\Big) + \sqrt{\frac{1}{m}\log\Big(\frac{1}{\delta}\Big)} \right),
\end{equation}
Recall that $\sigma:=\sup_{\vecw\in\W}\twonms{\mat{\Sigma}(\vecw)}^{1/2}$, and thus~\eqref{eq:subgaussian} implies that
\[
\frac{1}{\sqrt{m}} \twonms{\matG(\vecw)} \lesssim \frac{\sigma}{\sqrt{n}} + \frac{\zeta}{\sqrt{n}} \left( \sqrt{\frac{d}{m}} + \frac{d}{m} + \frac{1}{m}\log\Big(\frac{1}{\delta}\Big) + \sqrt{\frac{1}{m}\log\Big(\frac{1}{\delta}\Big)} \right)^{1/2}.
\]
Recall the $ \delta_0 $-net $\W_{\delta_0} = \{\vecw^1,\vecw^2,\ldots,\vecw^{N_{\delta_0}}\}$ as defined above. Then, we have with probability at least $1-\delta$, for all $\vecw^\ell \in \W_{\delta_0}$
\begin{equation}\label{eq:in_delta_net_op}
\frac{1}{\sqrt{m}} \twonms{\matG(\vecw^\ell)} \lesssim \frac{\sigma}{\sqrt{n}} + \frac{\zeta}{\sqrt{n}} \left( \sqrt{\frac{d}{m}} + \frac{d}{m} + \frac{1}{m}\log\Big(\frac{N_{\delta_0}}{\delta}\Big) + \sqrt{\frac{1}{m}\log\Big(\frac{N_{\delta_0}}{\delta}\Big)} \right)^{1/2}.
\end{equation}
For each $\vecw$ with $\twonms{\vecw^\ell - \vecw} \le \delta_0$, we have
\begin{align*}
\twonms{\matG(\vecw^\ell) - \matG(\vecw) } \le & \fbnorms{ \matG(\vecw^\ell) - \matG(\vecw)  } \\
\le & \left( \sum_{i=1}^m \twonms{(\gradF_i(\vecw^\ell) - \gradF(\vecw^\ell)) - (\gradF_i(\vecw) - \gradF(\vecw))}^2 \right)^{1/2} \\
\le & 2L\delta_0\sqrt{m}.
\end{align*}
This implies that when the bound~\eqref{eq:in_delta_net_op} holds, we have for all $\vecw\in\W$,
\begin{equation}\label{eq:unif_op_norm}
\frac{1}{\sqrt{m}} \twonms{\matG(\vecw)} \lesssim \frac{\sigma}{\sqrt{n}} + \frac{\zeta}{\sqrt{n}} \left( \sqrt{\frac{d}{m}} + \frac{d}{m} + \frac{1}{m}\log\Big(\frac{N_{\delta_0}}{\delta}\Big) + \sqrt{\frac{1}{m}\log\Big(\frac{N_{\delta_0}}{\delta}\Big)} \right)^{1/2} + 2L\delta_0.
\end{equation}
Choose $\delta_0 = \frac{1}{nmL}$, in which case the last term above is a high order term. In this case, choosing $\delta = \frac{1}{(1+mnDL)^d}$, we have with probability at least $1-\frac{1}{(1+mnDL)^d}$, for all $\vecw\in\W$,
\begin{align}
\frac{1}{\sqrt{m}} \twonms{\matG(\vecw)} \lesssim & \frac{\sigma}{\sqrt{n}} + \frac{\zeta}{\sqrt{n}} \left( \Big(\frac{d}{m} + \sqrt{\frac{d}{m}}\Big) \log(1+nmDL) \right)^{1/2} \nonumber \\
\lesssim & \frac{\sigma}{\sqrt{n}} + \frac{\zeta}{\sqrt{n}} \left( 1+\sqrt{\frac{d}{m}}  \right)\sqrt{\log(1+nmDL)}. \label{eq:unif_op_norm_2}
\end{align}
Combining the bounds~\eqref{eq:sufficient},~\eqref{eq:unif_grad_norm}, and~\eqref{eq:unif_op_norm_2}, we obtain that with probability at least $1-\frac{2}{(1+mnDL)^d},$
\[
\sup_{\vecw\in\W}\twonms{\vecg(\vecw) - \gradF(\vecw)} \lesssim \left( (\sigma + \zeta)\sqrt{\frac{\alpha}{n}} + \zeta\sqrt{\frac{d}{nm}} \right)\sqrt{\log(1+nmDL)},
\]
which completes the proof.

\subsection{Median and Trimmed Mean}\label{apx:med_tm}
In this section, we present the error bounds of median and trimmed mean operations in the Byzantine setting in~\cite{yin2018byzantine} for completeness.
\begin{asm}\label{asm:lip_each_loss}
For any $\vecz\in\Z$, the $ k $-th partial derivative $\partial_kf(\cdot;\vecz)$ is $L_k$-Lipschitz for each $k\in[d]$. Let $\widehat{L} := (\sum_{k=1}^d L_k^2)^{1/2}$. 
\end{asm}

For the median-based algorithm, one needs to use the notion of the \emph{absolute skewness} of a one-dimensional random variable $X$, defined as $S(X) := {\EXPS{|X -  \EXPS{X}|^3}}/{\var(X)^{3/2}}$. Define the following upper bounds on the standard deviation and absolute skewness of the gradients: 
\[
v := \sup_{\vecw\in\W} \big( \EXPS{\twonms{\gradf(\vecw; \vecz) - \gradF(\vecw)}^2} \big)^{1/2}, \quad
s := \sup_{\vecw\in\W}\max_{k\in[d]} S\big(\partial f_k(\vecw; \vecz)\big).
\]

Then one has the following guarantee for the median-based algorithm.

\begin{theorem}[median]\label{thm:median_inexactness}
~\cite{yin2018byzantine} Suppose that Assumption~\ref{asm:lip_each_loss} holds. Assume that 
\[
\alpha + \bigg( \frac{ d\log( 1 + nmD\widehat{L} ) }{ m(1-\alpha) } \bigg)^{1/2} + c_1\frac{ s }{\sqrt{n}} \le \frac{1}{2}-c_2
\]
for some constant $c_1,c_2>0$. Then, with probability $1 - o(1)$, $\mathsf{GradAGG}\equiv \med$ provides a $\Delta_{\med}$-inexact gradient oracle with
\[
\Delta_{\med} \le  \frac{c_3}{\sqrt{n}}v \big(\alpha + (\frac{d\log( n m D \widehat{L} )}{m})^{1/2} + \frac{s}{\sqrt{n}} \big) + \bigo(\frac{1}{nm}),
\]
where $c_3$ is an absolute constant.
\end{theorem}

Therefore, the median operation provides a $\widetilde{\bigo} ( v (\frac{\alpha}{\sqrt{n}} + \sqrt{\frac{d}{nm}} + \frac{s}{ n }) )$-inexact gradient oracle. If each partial derivative is of size $\bigo(1)$, the quantity $v$ is of the order $\bigo(\sqrt{d})$ and thus one has $\Delta_{\med} \lesssim \frac{\alpha \sqrt{d}}{\sqrt{n}} + \frac{d}{\sqrt{nm}} + \frac{\sqrt{d}}{n}$.

For the trimmed mean algorithm, one needs to assume that the gradients of the loss functions are sub-exponential.

\begin{asm}\label{asm:sub_exp_grad}
For any $\vecw \in \W$, $\gradf(\vecw; \vecz)$ is $\xi$-sub-exponential.
\end{asm}

In this setting, there is the following guarantee.

\begin{theorem}[trimmed mean]\label{thm:trmean_inexactness}
~\cite{yin2018byzantine} Suppose that Assumptions~\ref{asm:lip_each_loss} and~\ref{asm:sub_exp_grad} hold. Choose  $ \beta = c_4 \alpha \le \frac{1}{2} - c_5 $ with some constant $c_4 \ge 1$, $c_5 >0$. Then, with probability $1-o(1)$, $\mathsf{GradAGG} \equiv \trim_\beta$ provides a $\Delta_{\tm}$-inexact gradient oracle with
\[
\Delta_{\tm} \le c_6 \xi d \Big(\frac{\alpha}{\sqrt{n}} + \frac{1}{\sqrt{nm}} \Big) \sqrt{\log(n m D \widehat{L}) },
\]
where $c_6$ is an absolute constant.
\end{theorem}
Therefore, the trimmed mean operation provides a $\widetilde{\bigo} ( \xi d (\frac{\alpha}{\sqrt{n}} + \frac{1}{\sqrt{nm}} ) )$-inexact gradient oracle.

\subsection{Lower Bound for First-Order Guarantee}\label{apx:lower_bound}

In this section we prove Observation~\ref{obs:lower_bound}.
We consider the simple mean estimation problem with random vector $\vecz$ drawn from a distribution $\D$ with mean $\vecmu$. The loss function associated with $\vecz$ is $f(\vecw;\vecz) = \frac{1}{2}\twonms{\vecw-\vecz}^2$. The population loss is $F(\vecw) = \frac{1}{2}(\twonms{\vecw}^2 - 2\vecmu^\top\vecw + \EXPS{\twonms{\vecz}^2})$, and $\gradF(\vecw) = \vecw - \vecmu$. 
We first provide a lower bound for distributed mean estimation in the Byzantine setting, which is proved in~\cite{yin2018byzantine}. 

\begin{lemma}\label{lem:lower_bound_mean}
~\cite{yin2018byzantine} Suppose that $\vecz$ is Gaussian distributed with mean $\vecmu$ and covariance $\sigma^2\matI$. Then, any algorithm that outputs an estimate $\tvecw$ of $\vecmu$ has a constant probability such that 
\[
\twonms{\tvecw - \vecmu} = \Omega(\frac{\alpha}{\sqrt{n}} + \sqrt{\frac{d}{nm}} ).
\]
\end{lemma}
Since $\gradF(\tvecw) = \tvecw - \vecmu$, the above bound directly implies the lower bound on $\twonms{\gradF(\tvecw)}$ in Observation~\ref{obs:lower_bound}.

\end{document}